\documentclass[journal,comsoc]{IEEEtran}
\usepackage{cite}
\usepackage{amsmath,amssymb,amsfonts}
\usepackage{algorithmic}
\usepackage{graphicx,color,subfigure}
\usepackage{textcomp}
\usepackage{booktabs} % for much better looking tables
\usepackage{array} % for better arrays (eg matrices) in maths
\usepackage{paralist} % very flexible & customisable lists (eg. enumerate/itemize, etc.)
\usepackage{enumitem}
\usepackage{verbatim} % adds environment for commenting out blocks of text & for better verbatim
\usepackage[linesnumbered,ruled]{algorithm2e} 
\usepackage{tabularx}  
\usepackage{engord}
\usepackage{booktabs,multirow}
\usepackage{url}
\usepackage{comment}
\usepackage{nomencl}

\usepackage{float}
\usepackage{mathtools}
\usepackage{bbding}
\usepackage{amsthm}

\newtheorem{theorem}{Theorem}[section]

\newcommand{\tabincell}[2]{\begin{tabular}{@{}#1@{}}#2\end{tabular}}
\newcommand\cparagraph[1]{\vspace{0.6mm}\noindent\textbf{#1.}}

\setlist{leftmargin=4mm}

\usepackage[T1]{fontenc}% optional T1 font encoding

\ifCLASSINFOpdf
\else
\fi
\begin{document}
\title{Coordination of Drones at Scale:\\ 
Decentralized Energy-aware Swarm Intelligence \\ 
for Spatio-temporal Sensing}

\author{\IEEEauthorblockN{Chuhao Qin\IEEEauthorrefmark{1}, and 
Evangelos Pournaras\IEEEauthorrefmark{1}} \\

\IEEEauthorblockA{\IEEEauthorrefmark{1}School of Computing, University of Leeds, Leeds, UK}
\thanks{
% Manuscript received December 1, 2012; revised August 26, 2015. 
Corresponding author: Evangelos Pournaras (email: e.pournaras@leeds.ac.uk).}}

% \markboth{Journal of \LaTeX\ Class Files,~Vol.~14, No.~8, August~2015}%
% {Shell \MakeLowercase{\textit{et al.}}: Bare Demo of IEEEtran.cls for IEEE Communications Society Journals}



\maketitle

\begin{abstract}
Smart City applications, such as traffic monitoring and disaster response, often use swarms of intelligent and cooperative drones to efficiently collect sensor data over different areas of interest and time spans. However, when the required sensing becomes spatio-temporally large and varying, a collective arrangement of sensing tasks to a large number of battery-restricted and distributed drones is challenging. To address this problem, this paper introduces a scalable and energy-aware model for planning and coordination of spatio-temporal sensing. The coordination model is built upon a decentralized multi-agent collective learning algorithm (EPOS) to ensure scalability, resilience, and flexibility that existing approaches lack of. Experimental results illustrate the outstanding performance of the proposed method compared to state-of-the-art methods. Analytical results contribute a deeper understanding of how coordinated mobility of drones influences sensing performance. This novel coordination solution is applied to traffic monitoring using real-world data to demonstrate a $46.45\%$ more accurate and $2.88\%$ more efficient detection of vehicles as the number of drones become a scarce resource. 
\end{abstract}


% \fi
\IEEEpeerreviewmaketitle

\section{Introduction}
\IEEEPARstart{U}{nmanned} Aerial Vehicles (UAVs), commonly known as \emph{drones}, can form swarms for self-organization and collaboration in airborne mobile ad hoc networks. Nowadays, swarms of drones emerge in Smart Cites for spatio-temporal sensing~\cite{wu2016addsen,inoue2020satellite,butilua2022urban}. They are assigned tasks to execute such as collecting sensor data in areas of interest. For instance, swarms can capture images/videos of traffic-related information on public roadways; measure air temperature and humidity to support sustainable crop production; transmit real-time reports of natural disasters such as fire and car accidents; or accurately deliver goods in densely populated areas. In transportation system, drones can be used for an accurate monitoring of traffic to detect traffic congestion at early stage. This allows traffic operators to apply mitigation actions that decrease the carbon footprint of a sector with one of the highest carbon emissions worldwide~\cite{butilua2022urban,barmpounakis2020new}. 

To assist swarms of drones to complete sensing tasks efficiently, autonomous control of swarms and assignment of sensing tasks become a niche. Coordinated sensing involves the assignment of different sensing tasks to each drone while meeting the sensing requirements, drone capabilities and constraints~\cite{poudel2022task}. Earlier work is proposed to address the task assignment problem for efficient and large-scale spatio-temporal sensing by swarms of drones~\cite{fu2019secure,yanmaz2018drone,zhou2020uav}. Considering the heterogeneity and number of tasks, the problem is formulated as an NP-hard combinatorial optimization problem to find the optimal assignment of sensing tasks. Task assignment algorithms designed to solve spatio-temporal sensing problems range from market-based methods~\cite{alighanbari2005decentralized} to swarm intelligence~\cite{mac2018development}. 

However, mainstream approaches for UAV sensing task assignments do not achieve scalability, resilience and flexibility. This is because they rely on the centralized decision making while sacrificing a significant autonomy from drones~\cite{poudel2022task,bupe2015relief}. Therefore, the goal of this paper is to tackle instead a decentralized task assignment problem; drones collectively arrange and self-assign their sensing tasks, which is a highly complex research endeavor. On the one hand, coordinating the sensing tasks of drones is complex, i.e., large areas of interest, with varying sensing requirements and time-constrained missions. Certain areas with traffic jams or accidents may require for drones more fine-grained sensor measurements than areas with more regular traffic patterns. On the other hand, the inherent limitations in battery capacity influence spatio-temporal coverage. To tackle this complex task self-assignment problem, a decentralized and energy-aware coordination of drones at scale is introduced. Autonomous drones share information and allocate tasks cooperatively to meet complex sensing requirements while respecting battery constraints. Furthermore, the decentralized coordination method prevents single points of failure, it is more resilient and preserves the autonomy of drones to choose how they navigate and sense~\cite{nik2021using}.

A novel coordination model is designed using a multi-agent collective learning algorithm for multi-objective combinatorial optimization. It is selected because of its remarkable scalability (support of a large number of software agents/drones), efficiency (low communication and computational cost) and resilience, while preserving the privacy and autonomy of agents~\cite{pournaras2018decentralized,pournaras2020collective,Pournaras2020}. The self-assignment problem of sensing tasks by a swarm of coordinated UAVs is also validated within a novel prototyping testbed deployed in indoor sensing environments~\cite{qin20223,Fanitabasi2020}. Furthermore, a novel plan generation strategy with three policies is designed to produce effective navigation and sensing alternatives for flexible drones. A power consumption model~\cite{Stolaroff2018} is used in the proposed strategy to estimate the energy consumption of these alternatives. This can make the coordination of drones energy-aware as each drone determines its sensing tasks based upon its energy consumption restrictions, while improving the sensing performance. Finally, extensive evaluations are conducted using analytical methods, simulated and real-world data including comparisons to the state-of-the-art baseline methods.

The contributions of this paper are outlined as follows: 
\begin{enumerate}
	\item  A first study of the task self-assignment problem for spatio-temporal sensing by a swarm of interactive drones;
	\item  A decentralized coordination model by integrating multi-agent collective learning~\cite{pournaras2018decentralized} to improve scalability, resilience, and flexibility of spatio-temporal sensing;
	\item  A plan generation strategy with three policies based on a power consumption model~\cite{Stolaroff2018} to make the coordination model energy-aware and achieve a highly efficient navigation and sensing of drones.
    \item An open dataset~\cite{qin_data2022} containing all the plans of the studied scenario. They can be used as benchmarks to encourage further research on this problem. 
    \item Analytical results to understand more rigorously how coordinated mobility of drones influences sensing performance. 
    \item A comprehensive empirical understanding of how a large spectrum of factors such as the number of dispatched drones, the spatial granularity of sensing, the number of base stations and the required amount of collected data influence the sensing performance. 
	\item A traffic monitoring model using coordinated drones validated as accurate and efficient using real-world data. 
\end{enumerate}

The rest of this paper is organized as follows: Section~\ref{sec:related} positions this work in literature. Section~\ref{sec:model} introduces the designed coordination model of spatio-temporal sensing. Section~\ref{sec:method} illustrates our proposed plan generation method. Section~\ref{sec:evaluation} illustrates the experimental evaluation and Section~\ref{sec:conclusion} concludes this paper. The appendices contain the power consumption model, the analytical results with proofs of theorems and further information on the experimental settings. 

\section{RELATED WORK} \label{sec:related}
The assignment of sensing tasks within a swarm of intelligent and cooperative drones includes applications of traffic monitoring, disaster response, smart farming, last-mile delivery, etc.~\cite{poudel2022task}. The UAV task assignment problem is earlier defined as a Traveling Salesmen Problem for combinatorial optimization~\cite{gao2018multi,wu2021multi,chen2018multi}. An optimal task assignment for drones is found at specified places, subject to constraints including task emergency, time scheduling and flying costs. The digraph-based methods are introduced to formulate the problem of reaching optimal sensing efficiency in terms of coverage, inspection delay, events detection rate and the cost of flying trajectories. Nevertheless, existing models do not address the energy impact of task assignment since small-scale spatio-temporal scenarios do not drain the batteries of drones significantly. Although earlier work focuses on solving the energy-aware task assignment problem for large-scale and efficient sensing, it relies on a centralized system and does not consider the time of sensor data collection~\cite{motlagh2019energy,zhou2018mobile,bartolini2020multi}. This paper models the cost of plans calculated by the total power consumption of flying and hovering~\cite{Stolaroff2018}, which is used to optimize the decentralized coordination of swarms, while accounting for the battery constraint of each drone.

\begin{table}
% \fontfamily{times}\selectfont
	\centering
	\caption{Comparison to related work: (\Checkmark) indicates criteria covered, (\XSolid) indicates criteria not covered.}  
	\label{table:criteria}
	\resizebox{\linewidth}{!}{
	\begin{tabular}{lcccccccc}  
		\toprule  
		\textbf{Criteria \, Ref.:} &\tabincell{l}{\cite{gao2018multi}} &\tabincell{l}{\cite{wu2021multi}} &\tabincell{l}{\cite{chen2018multi}} &\tabincell{l}{\cite{bartolini2019task}} &\tabincell{l}{\cite{alighanbari2006robust}} &\tabincell{l} {\cite{chen2022consensus}} &\tabincell{l} {\cite{elloumi2018monitoring}} &This paper\\  
		\midrule  
		Sensing-efficiency 	&\Checkmark &\Checkmark &\Checkmark &\Checkmark &\Checkmark &\Checkmark &\Checkmark &\Checkmark \\   
		Energy-awareness		&\XSolid	&\XSolid	&\XSolid	&\Checkmark	&\XSolid	&\XSolid  &\XSolid	&\Checkmark \\
		Scalability 		&\XSolid	&\Checkmark	&\Checkmark	&\Checkmark	&\Checkmark	&\Checkmark &\XSolid	&\Checkmark \\  
		Resilience			&\XSolid	&\XSolid  	&\XSolid 	&\XSolid	&\XSolid &\XSolid &\XSolid	&\Checkmark	 \\  
		Flexibility		&\XSolid	&\XSolid	&\XSolid	&\XSolid	&\Checkmark	&\Checkmark &\XSolid	&\Checkmark \\  
  		Coordination		&\XSolid	&\XSolid	&\XSolid	&\XSolid	&\Checkmark	&\Checkmark &\XSolid	&\Checkmark \\
		\bottomrule
	\end{tabular}  
	}
\end{table} 

In the view of the optimization approaches shown in Table~\ref{table:criteria}, scholars introduce algorithms that require a centralized computation, such as particle swarm optimization~\cite{gao2018multi}, genetic programming~\cite{wu2021multi} and wolf pack search~\cite{chen2018multi}. Chen \textit{et al}.~\cite{chen2018multi} introduce a deadlock-free algorithm to prevent two or more drones from waiting for each other in a simultaneous task, and leverage the classical interior point method and wolf pack search algorithm to solve the uncertainty problem, including the uncertainty of the flying velocity, task effectiveness and communication. Wu \textit{et al}.~\cite{bartolini2019task} leverage a greedy approach by removing redundant target points from trajectories to maximize the weighted coverage while respecting energy constraints (with 5 drones to cover 225 target points). However, these centralized methods face the risk of single point of failure~\cite{omidshafiei2017decentralized}. If the central control station 
fails, the task assignment process cannot recover or mitigate such failure.

Distributed task assignment algorithms include the robust decentralized task assignment (RDTA)~\cite{alighanbari2006robust} and the consensus-based bundle algorithm  (CBBA)~\cite{chen2022consensus}. RDTA improves robustness and reduces communication cost using the decentralized planning, wherein the drones plan independently and communicate a set of candidate plans. CBBA combines the distributed structure of market-based mechanisms and situation awareness with a consensus strategy that converges and can avoid conflicts between the tasks executed by drones. Both algorithms coordinate autonomous drones, and can ensure scalability and flexibility of multi-UAV systems to adapt to complex sensing scenarios. Nevertheless, they do not address the constraints of energy consumption. Moreover, the resilience of the algorithms is not verified, in contrast to earlier work on EPOS shown to preserve its learning capacity in dynamic and unstable networked environments~\cite{Pournaras2020}. The proposed method in this paper overcomes scalability and resilience barriers, while leaving drones with a significant level of autonomy to make a flexible cooperative selection of tasks to execute. Furthermore, the prototyping problem of moving from simulation, to live deployment of decentralized socio-technical systems, and ultimately to a robust operation of a high technology readiness level (TRL) as well as online iterative traffic optimization is addressed in earlier work, with EPOS as a case study~\cite{fanitabasi2020self,gerostathopoulos2019trapped}.

Although drones have shown a potential to advance traffic monitoring in accuracy, safety and cost savings~\cite{outay2020applications}, there is a very limited related work that focus on the task assignment in this scenario. Elloumi \textit{et al}.~\cite{elloumi2018monitoring} propose a road traffic monitoring system that generates adaptive UAVs trajectories by extracting information about vehicles. Although this approach shows high performance in collecting traffic vehicle data, it does not address energy constraints and varying tasks over different time spans. 

\begin{table}  
	\caption{Notations.}  
	\centering
	\begin{tabularx}{\linewidth}{lXl}  
		\hline  
		Notation & Explanation \\  
		\hline    
		$ \mathcal{U}, \mathcal{N} $  & The set of dispatched drones; the set of grid cells \\  
		$ u, U $  & The index of a dispatched drone; the total number of dispatches \\  
		$ n, N $  & The index of a grid cell; the total number of grid cells \\ 
		$ T_n $  & The target value (or sensing requirements) at cell $n$ \\  
		$ S_{u,n} $ & The sensing values collected by drone $u$ at cell $n$  \\
		$ \tau_u $ & The total flying time of drone $u$ without hovering (seconds) \\
		$ P^\mathsf{f}_u, P^\mathsf{h}_u $ & Flying and hovering power consumption of drone $u$ (Watt) \\
  		$ E^\mathsf{f}_u, E^\mathsf{h}_u $ & Flying and hovering energy consumption of drone $u$ (Joule) \\
		$ C_u, e $ & The battery capacity of drone $u$ (Joule); energy utilization ratio \\
		$ J_u $ & The set of visited cells \\
        $ f $ & The frequency with which drones collect sensing data (unit of sensing value per second) \\
        $ p, P $  & The index of a plan; the total number of plans \\
        $ B $  & The index of a base station for departure/return \\
        $ K_B $ & The travel range or cells close to base station $B$\\
        $ m, M $  & The index of a time unit in a plan; the total number of time units in a plan \\ 
        $ h_{m,n} $ & The element of the plan to represent whether a drone occupies the cell $n$ at time unit $m$ \\
        $ V_m, V^*_m $  & The total number of vehicles at time unit $m$; the number of vehicles observed by drones at time unit $m$ \\ 
		\hline  
	\end{tabularx}  
	\label{table:notation}
\end{table} 

\section{SYSTEM MODEL} \label{sec:model}
This section introduces the coordination model to solve the decentralized task assignment problem for spatio-temporal sensing by a swarm of drones. Table~\ref{table:notation} illustrates the list of mathematical notations used in this paper.

\begin{figure*}[h]
	\centering
	\includegraphics[width=\linewidth]{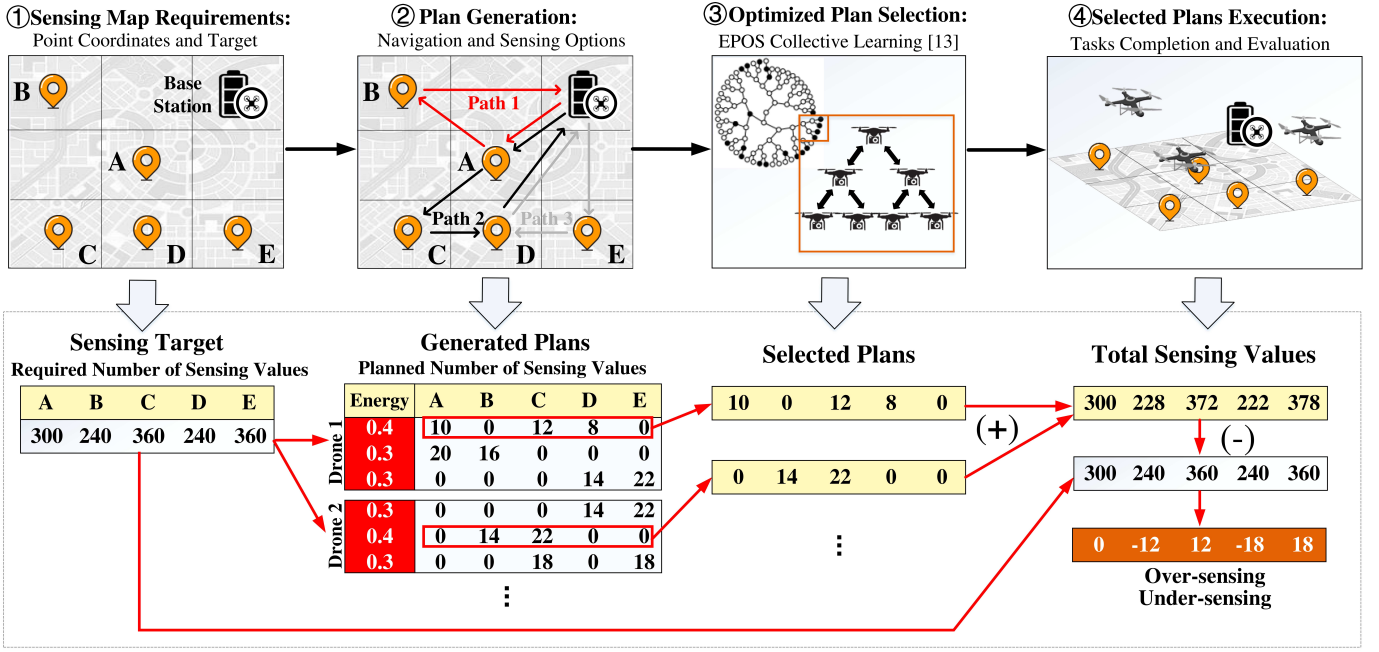}
	\caption{Overview of the decentralized energy-aware coordination model for spatio-temporal sensing by a swarm of drones.}
	\label{fig:overview}
\end{figure*}

\subsection{Scenario and architecture}
Consider a swarm of drones that perform sensing over a grid — a 2D map over the spatial illustrative model shown in Fig.~\ref{fig:overview}. A dispatch $u$ is defined as a sensing task between the departure and return of a drone. In this scenario, a finite number of base/charging stations, from where drones depart and return, are set with fixed coordinates in the area. Drone are equally distributed over these base stations, and have to return to the original base stations from which they depart. In addition, let $\mathcal{N}$ be the set of $N$ points of interest, each is regarded as a grid cell that covers an area in the map. The main goal of drones is to coordinate their visits to cells to collect the required data.

\cparagraph{Sensing map requirements}
In the context of a sensing task, each cell at a time period has specific sensing requirements that determine the hovering duration and data acquisition of drones. Such sensing requirements can be determined by a continuous kernel density estimation, for instance, monitoring cycling risk based on requirements calculated by past bike accident data and other information~\cite{castells2020cycling,qin20223}. A high risk level in a cell represents higher sensing requirements over this area (e.g., the accidents or crucial intersection of traffic flow). Thus, a higher number of sensing values is set at this cell, which requires drones to hover a longer time over the cell to measure accurately. For example, see Fig.~\ref{fig:overview}. It is assumed that the cell $A$ requires the total sensing values of $300$ at a time period. The sensing requirements of cells at a time period are set as the \textit{targets} encoded by a vector of size $N$, e.g., $\{300,240,360,240,360\}$.

\cparagraph{Plan generation}
To coordinate the assignment of sensing tasks, the drones use the decentralized multi-agent collective learning method of EPOS, the \textit{Economic Planning and Optimized Selections}~\cite{pournaras2018decentralized,pournaras2020collective}. Given the sensing map, each drone, controlled by a local software agent, autonomously generates a finite number of discrete navigation and sensing alternatives, which provide flexibility for the drones to choose in a coordinated way. Each alternative has sensing details as well as an estimated energy consumption; the alternatives are referred to as the \textit{possible plans} and each comes with a (normalized) \textit{cost} respectively. For instance, a plan encoding that a drone travels and hovers over the cells $A$, $C$ and $D$ to collect respectively $12$, $10$ and $10$ sensing values over a time period is encoded with: $\{12,0,10,10,0\}$ with $N=5$. The energy consumed by a drone that carries out its planned tasks is calculated via a power consumption model~\cite{Stolaroff2018} with input the specification of the drone (weight, propeller, and battery parameters). This model can estimate the cost of navigation and sensing plans, and emulates the outdoor environments~\cite{qin20223}.

\cparagraph{Plan selection and execution}
To make coordinated plans selection, the agents of the drones connect into a tree communication structure within which each interacts with its children and parent in a bottom-up and top-down fashion to improve plan selections iteratively~\cite{pournaras2018decentralized}. The objective of this method is to select the optimal plan for each agent such that all choices together add up to maximize the sensing quality: the overall sensing data collected by all agents matches well the required data (target, see Fig.~\ref{fig:overview}). In contrast, the mismatch is a result of \textit{over-sensing} and \textit{under-sensing}. For instance, in Fig.~\ref{fig:overview}, the sensing value of $372$ for which drones hover over the cell $C$ exceeds the requirements of $360$  (over-sensing), whereas drones hover over the cell $E$ to collect $222$ sensor values that is lower than the requirements of $240$ (under-sensing). Over-sensing causes excessive data that needs further processing, waste of energy consumption, high storage and privacy cost, while under-sensing fails to satisfy sensing requirements. As a consequence, the generation and selection of high-quality plans are required to eliminate over-sensing and under-sensing.

\subsection{Problem formulation}
The goal of the proposed system is to find the optimal plan for each drone $u \in \mathcal{U}$, i.e., determining which cells a drone visits and how many sensor values it collects over a time period. Apart from preventing over-sensing and under-sensing, the formulated problem aims to maximize the accomplishment of sensing tasks under the energy constraints for drones. Therefore, the problem is formulated as follows:

% \begin{equation}
% 	\begin{split}
%         \mathop{\min}\limits_{S_{u,n}, u \in \mathcal{U}, n \in \mathcal{N}}
% 		& \sum_{n \in \mathcal{N}} (T_n - \sum_{u \in \mathcal{U}} S_{u,n})^2\\
% 		& \cdot (1 - \frac{ \sum_{u \in \mathcal{U}} \sum_{n \in \mathcal{N}} S_{u,n} }{\sum_{n \in \mathcal{N}} T_n}),
% 	\end{split}
% 	\label{eq:model_optimal}
% \end{equation}
\begin{align}
    \begin{split}
        \mathop{\min}\limits_{S_{u,n}, u \in \mathcal{U}, n \in \mathcal{N}}
		&  \sum_{n \in \mathcal{N}} (T_n - \sum_{u \in \mathcal{U}} S_{u,n})^2,
    \end{split}
    \label{eq:model_optimal} \\
        \begin{split}
         \mathop{\min}\limits_{S_{u,n}, u \in \mathcal{U}, n \in \mathcal{N}}
        & (1 - \frac{ \sum_{u \in \mathcal{U}} \sum_{n \in \mathcal{N}} S_{u,n} }{\sum_{n \in \mathcal{N}} T_n}),
    \end{split}
    \label{eq:model_optimal2} \\
    s.t. \quad  
    & \sum_{u \in \mathcal{U}} \sum_{n \in \mathcal{N}} S_{u,n} \leq \sum_{n \in \mathcal{N}} T_n, \label{eq:sense} \\ 
    &P^\mathsf{f}_u \cdot \tau_u + P^\mathsf{h}_u \cdot \frac{\sum_{u \in \mathcal{U}} S_{u,n}}{f} \leq C_u. \label{eq:energy_consump}
\end{align}

\noindent where $S_{u,n}$ is the sensor values collected by the drone $u$ at cell $n$, $u \in \mathcal{U}$, $n \in \mathcal{N}$; $T_n$ is the target value at cell $n$; Objective (\ref{eq:model_optimal}) measures the over-sensing and under-sensing using a quadratic function~\cite{allen2016variance}. Objective (\ref{eq:model_optimal2}) measures the accomplishment of sensing tasks. Equation (\ref{eq:sense}) limits the total sensor value collected by drones to the required one at maximum. Constraint (\ref{eq:energy_consump}) models the energy constraint of the drone $u$, where $P^\mathsf{f}_u$ and $P^\mathsf{h}_u$ are the flying and hovering power consumption respectively~\cite{Stolaroff2018}; $\tau_u$ is the total flying time that drone $u$ travels without hovering; $f$ is the frequency with which drones collect sensor data as they hover over a cell.

EPOS optimizes sensing quality via a coordinated selection among alternative plans generated locally by the drones (i.e., navigation and sensing options). However, plans generation also results in a series of new problems, including \textit{how to select the cells}, \textit{how to determine collected sensing values}, \textit{how to calculate the energy cost of a plan}, and \textit{how to prevent more than two drones to occupy the same cell at the same time}.  Therefore, this paper proposes a novel plan generation strategy to solve these problems such that drones coordinate efficiently to achieve high-quality sensing.

\begin{algorithm} 
	%    \SetAlgoNoLine  
	\caption{The local plan generation strategy for each drone.} 
	\label{algorithnm1}
	\KwIn{Power consumption of drone $u$ for flying $P^\mathsf{f}_u$ and hovering $P^\mathsf{h}_u$, the battery capacity $C_u$, the targets $\boldsymbol{T} = (T_{1},...,T_{N})$ and their coordinates, the base station $B$ for departure/return and its travel range $K_B$, the total number of visited cells $|J_u|$, the number of generated plans $ P $.}
	\textbf{Initialization}: Initialize a set of plans $\widehat{\mathbb{P}} = \emptyset$\;
	% \cparagraph{Travel range calculation} Determine the total number of visited cells $|J_u|$ based on one of the three policies: (1) \emph{min sensing mismatch}, (2) \emph{min mission inefficiency}, and (3) \emph{balance}\;
	\For{each plan index $p := 1,...,P$}
	{
		\textbf{Path calculation}: Find $J_u$ via the K-nearest search within the range $K_B$\;	
		Find the shortest path among visited cells and base station $B$ via the greedy algorithm\;
		\textbf{Energy calculation}: Calculate the flying energy $E^\mathsf{f}(J_u)$ based upon the path;
		Determine the energy utilization ratio $e$ based on $p$ via Eq.(\ref{eq:energy_ratio})\;
		Calculate the maximum hovering energy $E^\mathsf{h}(J_u)$ via Eq.(\ref{eq:hover_energy})\;
		Calculate the total sensing values $S(J_u)$ to collect via Eq.(\ref{eq:total_sense}) and the total targets $T(J_u)$ via Eq.(\ref{eq:total_target})\;
		\textbf{Sensing allocation}: Allocate the sensing values $S_{u,n}$ proportionally to the visited cells via Eq.(\ref{eq:sense_allocation})\;
		\textbf{Plan generation}: Calculate the cost of the plan $E(J_u)$ via Eq.(\ref{eq:total_energy})\;
		Generate the plan $\mathbb{P}$ of size $N \times M$, and add it to the set $\widehat{\mathbb{P}}$ of sensing plans with $E(J_u)$\;
	}
	\KwOut{Set of plans $\widehat{\mathbb{P}}$ for drone $u$.}
\end{algorithm} 

\section{PROPOSED METHOD}\label{sec:method}
Algorithm 1 outlines the following steps of the proposed plan generation strategy for a swarm of drones.

\cparagraph{Initialization}
Given a drone $u$, the inputs of the algorithm are listed: The flying and hovering power consumption are calculated based on the power consumption model~\cite{Stolaroff2018}. The parameters of drone $u$ $\{ m_b, m_e, d, r, v, F_d, \epsilon \}$ and environmental parameters such as air density and air speed are determined (see Table~\ref{table:environment} and Appendix A). Each drone $u$ also needs the information of the map including the coordinates and sensing requirements of cells, as well as the base stations of departure and return. 

Next, according to the objective functions of Eq.(\ref{eq:model_optimal}) and (\ref{eq:model_optimal2}), the total number of visited cells $|J_u|$ is determined empirically (see Section~\ref{sec:param}) using one of the three policies: (1) \emph{min sensing mismatch}, (2) \emph{min mission inefficiency}, and (3) \emph{balance}. Min sensing mismatch focuses on avoiding over-sensing and under-sensing, while min mission inefficiency minimizes the uncollected sensing data. The policy for balance is a trade off between the first two policies. The policy of min sensing mismatch has a larger number of visited cells, while min mission inefficiency has a lower one. The two theorems in Appendix B provide the theoretical foundations behind the design of these policies. The algorithm also initializes the set $\widehat{\mathbb{P}}$ for the plans of drone $u$. 

\cparagraph{Path calculation}
At the beginning of each round, the algorithm finds the set of visited cells indices $J_u$ via the K-nearest search. It selects the first cell randomly from $K_B$ (a range of cells close to base station $B$). This range is calculated based on the relative distance between base stations. Then, the algorithm searches the nearest cell (within $K_B$) to the previous selected one until finding other $|J_u|-1$ cells. After this, the algorithm finds the shortest possible path among the cells of $J_u$ and the base station $B$ via the greedy algorithm for traveling salesmen problem~\cite{kizilatecs2013nearest}, and calculates the traveling time $\tau_u$ excluding hovering. Note that the drone returns to $B$ at the end of period, and thus the path starts and ends at $B$. Taking an example in Fig.\ref{fig:overview}, the first plan of Drone $1$ has the set of visited cells $\{A,C,D\}$. 

\cparagraph{Energy calculation}
The algorithm determines the maximum energy consumption $E^{max}_u$ of the drone before calculating the collected sensor values and the corresponding energy cost. The algorithm uses the energy utilization ratio $e$ to compute the maximum energy constraint $C_u \cdot e$. The ratio can be expressed as:
%\begin{equation}
%	e = \begin{cases}
%		1.00 & \textrm{$0 < p \leq P/4 $}\\
%		0.95 & \textrm{$P/4 < p \leq P/2 $}\\
%		0.90 & \textrm{$P/2 < p \leq 3P/4 $}\\
%		0.85 & \textrm{$3P/4 < p \leq P $}.
%	\end{cases}
%	\label{eq:energy_ratio}
%\end{equation}
\begin{equation}
	e = 1 - \frac{p}{\delta \cdot P},
	\label{eq:energy_ratio}
\end{equation} 
\noindent where $\delta$ is a constant to determine energy utilization. The ratio limits the energy consumption of drones, and gives them flexibility to select plans with varying cost, i.e. energy consumption. Next, the algorithm calculates the hovering energy consumption $E^\mathsf{h}(J_u)$ using the flying power consumption $P^\mathsf{f}_u$:
\begin{equation}
\begin{split}
    	E^\mathsf{h}(J_u) &= C_u \cdot e - E^\mathsf{f}(J_u),\\
                            &= C_u \cdot e - P^\mathsf{f}_u \cdot \tau(J_u).
\end{split}
	\label{eq:hover_energy}
\end{equation}
\noindent The total sensor values to collect among $J_u$ visited cells $S(J_u)$ is then calculated as follows:
\begin{equation}
	S(J_u) = \sum_{n \in \mathcal{N}} S_{u,n} = \frac{E^\mathsf{h}(J_u)}{P^\mathsf{h}_u} \cdot f,
	\label{eq:total_sense}
\end{equation}
\noindent where $P^\mathsf{h}_u$ is the power consumption for hovering. Both $P^\mathsf{f}_u$ and $P^\mathsf{h}_u$ are computed by the power consumption model~\cite{Stolaroff2018} (see Appendix A). 

\cparagraph{Sensing allocation}
To determine $S_{u,n}$ in each cell $n$, the algorithm allocates the collected sensor values among $J_u$ visited cells proportionally to the target; a higher number of sensor values is collected from the cell with a higher target value. The equation is shown as follows:
\begin{equation}
\begin{split}
        S_{u,n} = \left \{
        \begin{array}{ll}
           S(J_u) \cdot \frac{T_n}{T(J_u)},  &  n \in J_u\\
           0,  &    otherwise
        \end{array},
        \right.
\end{split}
    	\label{eq:sense_allocation}
\end{equation}
\begin{equation}
	where \quad T(J_u) = \sum_{n \in J_u} T_n.
	\label{eq:total_target}
\end{equation}
\noindent The proportional sensing allocation aims to improve the matching between the total sensor values collected and the required ones. We also prove its high performance by comparing it to the equal allocation (mean) $S_{u,n} = \frac{S(J_u)}{|J_u|}$ (see Section~\ref{sec:param}).  

\cparagraph{Plan generation}
With the energy consumption of hovering and flying, the algorithm computes the cost of the plan $E(J_u)$: 
\begin{equation}
    E(J_u) = C_u \cdot e.
    \label{eq:total_energy}
\end{equation}
% \begin{equation}
% 	E(J_u) = P^\mathsf{f}_u \cdot \tau(J_u) + P^\mathsf{h}_u \cdot \sum_{n \in \mathcal{N}} S_{u,n} \cdot f.
% 	\label{eq:total_energy}
% \end{equation}
\noindent Furthermore, the algorithm avoids more than two drones occupying the same cell at the same time. It generates a plan encoded by a matrix of size $N \times M$, where $M$ denotes the total number of time units. Each element in the matrix can be represented by $h^u_{m,n} \in \{0,1\}$, $m = 1,2,...,M$, $n = 1,2,...,N$, $\sum_u h^u_{m,n} \leq 1$ $\sum_{m} h^u_{m,n} \leq M$, $\sum_{n} h^u_{m,n} \leq N$, where $h^u_{m,n}=1$ denotes the drone occupies the cell $n$ at time unit $h$, whereas $h^u_{m,n}=0$ denotes the drone does not occupy the cell. The targets in EPOS are also set as a matrix of size $N \times M$, whose elements are set as 1 to prevent more than two drones occupying the same cell at same time. Finally, the algorithm adds the plan and its energy cost to the set $\widehat{\mathbb{P}}$.

\section{EXPERIMENTAL METHODOLOGY}\label{sec:evaluation}

This section introduces the experimental settings including the sensing map, the drones and the learning algorithm at first. The metrics and baselines used for the performance evaluation are also introduced. 

\subsection{Experimental settings}

\cparagraph{Sensing map}
A square area of size $1600 \times 1600$ meters split into a finite number of cells is studied. Each cell is defined as a rectangular square that can be captured by the cameras of drones (see Fig.~\ref{fig:overview}). The target values, or the sensing requirements of hovering time are distributed according to a Beta distribution. The base (charging) stations are uniformly distributed in the map. There are three types of scenarios for the experimental evaluation of this paper:
\begin{itemize}
	\item \emph{Basic synthetic scenario.} It has $4$ base stations, $64$ cells and the total target values of $20000$ to collect over $48$ time periods, which correspond to one day. Each period lasts $30$ min and is divided into $12$ time units, each of equal length. We dispatch $1000$ drones over all periods, during which approximate $20$ drones sense the area (camera recording) in parallel. The purpose to use this map is to compare the performance of different plan generation policies in the proposed method.
	\item \emph{Complex synthetic scenario.} It varies the parameter settings such as the number of dispatches, the number of base stations and cells as well as the total target values. The goal is to assess the scalability of the proposed method in different experimental conditions.
	\item \emph{Transportation scenario.} It originates from the central business district of Athens, where a swarm of drones use cameras to record traffic flows. The goal is to assess the accuracy and efficiency of the proposed method in the real-world traffic monitoring. More details are given in Part~\ref{sec:real} of this section.
\end{itemize}

\begin{table}  
	\caption{Notations for drones~\cite{Stolaroff2018}.}  
	\centering
	\begin{tabularx}{\linewidth}{lll}  
		\hline  
		Notation & Description & Value \\  
		\hline    
		$ m_{b} $ & mass of drone body & $1.07 kg$  \\
		$ m_{e} $ & mass of battery & $0.31 kg$  \\
		$ d $ & diameter of propellers & $0.35 m$  \\
		$ r $ & number of  propellers & $4$  \\
		$ v $ & ground speed & $6.94 m/s$  \\
		$ F_d $ & drag force & $4.1134 N$  \\
		$ \epsilon $ & power efficiency & $0.8$  \\
		$ C $ & battery capacity & $275 kJ$  \\
        $ f $ & sensing frequency & $60 \ seconds$  \\
		\hline  
	\end{tabularx}  
	\label{table:environment}
\end{table} 

\begin{table}  
	\caption{Parameters of the EPOS algorithm~\cite{pournaras2018decentralized}.}  
	\centering
	\begin{tabularx}{\linewidth}{lXl}  
		\hline  
		Parameters & Value \\  
		\hline  
		Number of agents/drones & $1000$  \\  
		Number of plans per agent & $64$  \\
		Network communication topology & balanced binary tree \\
		Number of repetitions & $40$  \\
		Number of iterations & $40$  \\
		Non-linear cost function & Min RSS (Unit-Length)  \\
        Energy utilization parameter & $\delta = 8$ \\
		Behavior of agent & $\beta = 0$  \\
		Number of tested maps & $200$  \\
		\hline  
	\end{tabularx}  
	\label{table:algorithm}
\end{table} 

\cparagraph{Drones}
We assume that drones are of the same type (DJI Phantom 4 Pro), equipped with the same type of battery (6000 mAh LiPo 2S) and camera (4K) to capture images/videos~\cite{Phantom_UAV}, and thus they have the same power consumption~\cite{Stolaroff2018} and battery capacity. To ensure the camera of a drone covers the whole area of a cell (see Fig.~\ref{fig:overview}), the minimum hovering height of drones is determined at which the field of view of the camera and the cell overlap. Based on the distance between any two cells $D$ (approximately $200$ meter in the basic synthetic scenario), the hovering height $H$ is computed using the pixels $\mathsf{PX}$, focal length derived from the camera calibration $c_k$ and ground sampling distance $\mathsf{GSD}$, with the formula: $H=\mathsf{GSD} \cdot c_k / \mathsf{PX}$~\cite{wierzbicki2018multi}. Thus, in this paper, each drone is equipped with a $4K$ camera, sensing from a minimum height of $82.4$ meter based on the pixels and field of view that a $4K$ camera has. The drone parameters and their description are summarized in Table~\ref{table:environment}. 

\cparagraph{Learning algorithm}
We generate plans for each of the $1000$ agents. All generated plans are made openly available to encourage further research on coordinated spatio-temporal sensing of drones~\cite{qin_data2022}. Each agent in EPOS is mapped to a drone. During the coordinated plan selection via EPOS\footnote{EPOS is open-source and available at: https://github.com/epournaras/EPOS.}, agents self-organize into a balanced binary tree as a way of structuring their learning interactions~\cite{pournaras2018decentralized}. The shared goal of the agents is to minimize the residual sum of squares (RSS) between the total sensor values collected and the total required target values, both in unit-length scaled. The RSS is a proxy to optimize the two objectives defined in Eq.(\ref{eq:model_optimal}) and (\ref{eq:model_optimal2}), and can be formulated as: $\log_{10} (\sum_{n \in \mathcal{N}} (\sum_{u \in \mathcal{U}} S_{u,n} - T_n)^2)$. The algorithm repeats $40$ times by changing the random position of the agents on the tree\footnote{More information about the influence of the tree topology and agents' positioning on the tree is illustrated in earlier work~\cite{Nikolic2019,Pournaras2020}.}. At each repetition, the agents perform $40$ bottom-up and top-down learning iterations during which RSS converges to the minimum optimized value. For the validation of the proposed algorithm, a number of $1000$ basic synthetic scenarios are generated, each with a new distribution of cells and target values. The algorithm requires at least $200$ sensing maps as input to output stable results (see Fig.~\ref{fig:statistics} in Appendix C). Therefore, a number of $200$ random tested maps are set for the following experiments. The computational cost of EPOS is $O(P \cdot I)$ for each agent and $O(P \cdot I \cdot \log U)$ for the system, where $I$ is the total number of iterations. The communication cost is $O(I)$ for each agent and $O(I \cdot \log U)$ for the system. Earlier comparisons of this complexity with state-of-the-art approaches demonstrate the cost-effectiveness of EPOS~\cite{pournaras2018decentralized}. Table~\ref{table:algorithm} summarizes the optimization parameters of EPOS.

\subsection{Metrics and baselines}
To evaluate the sensing quality, i.e., the accomplishment of sensing targets, we introduce two performance metrics:
\begin{itemize}
    \item \emph{Sensing mismatch.} It denotes the approximation error\footnote{Error and correlation metrics such as the root mean squared error, cross-correlation or residuals of summed squares can estimate this matching, which is shown earlier to be an NP-hard combinatorial optimization problem in this context~\cite{pournaras2018decentralized,pournaras2020collective}.} between the total sensing values collected and the target values. A low sensing mismatch prevents the cases of over-sensing and under-sensing~\cite{qin20223}. It is formulated as $log_{10}$ of Eq.(\ref{eq:model_optimal}).
    \item \emph{Mission inefficiency.} It denotes the ratio of sensing values in all cells that are not collected by the drones during their mission over the total target values in all cells. It is formulated as Eq.(\ref{eq:model_optimal2}).
\end{itemize}

In simple words, the sensing mismatch measures the data sampling quality, while the mission inefficiency measures the completeness of the required collected data. 

A fair comparison of the proposed method with related work is not straightforward as there is a very limited number of relevant decentralized algorithms. These algorithms~\cite{alighanbari2006robust,chen2022consensus} cannot be directly applied to this large-scale task assignment problem while respecting the energy constraints of drones. For this reason, we use as baselines for comparison three state-of-the-art centralized sensing methods capable of performing multi-drone task optimization: (1) \emph{Greedy-sensing}~\cite{bartolini2019task}, (2) \emph{Round-robin}~\cite{alwateer2019two} and (3) \emph{Min-energy}~\cite{pournaras2018decentralized}.
\begin{itemize}
	\item \emph{Greedy-sensing.} It requires a drone to complete the required sensing tasks of the cells one by one without violating the battery constraint. This method reduces the number of visited cells and traveling distance compared to the proposed method; drones spend more energy on sensing than traveling. We implement the method on a centralized coordinator that has a global view of the remaining sensor values required such that over-sensing and under-sensing are prevented. Table~\ref{table:greedy_sensing} illustrates the higher performance of the method with a global view vs. a version with a local view, i.e., no knowledge of the remaining sensor values required.
	\item \emph{Round-robin.} It comes in sharp contrast to \emph{Greedy-sensing}. Drones visit the same number of cells and spend more energy on traveling than \emph{Greedy-sensing}. According to the results shown in Table~\ref{table:round_robin}, the number of visited cells is divided equally into 8 for each drone as it has the minimum sensing mismatch.
	\item \emph{Min-energy.} It minimizes the total energy consumption and does not sacrifice energy for improving sensing quality. This method is implemented by changing the behavior of EPOS agents to $\beta=1$ such that the agents select the plans with the lowest energy consumption cost. No coordination is performed in this case.
\end{itemize}

\begin{table}[htbp]
	\centering
	\caption{Performance of two implementations in \emph{Greedy-sensing}.}  
	\label{table:greedy_sensing}
	\resizebox{6.5cm}{!}{
		\begin{tabular}{lll}  
			\toprule  
			Implementation &Global view&Local view\\  
			\midrule  
			Sensing Mismatch &2.79&3.21 \\    
			Mission Inefficiency (\%) &19.80&26.11 \\  
			\bottomrule
		\end{tabular}  
	}
\end{table} 

\begin{table}[htbp]
	\centering
	\caption{Results for the number of visited cells in \emph{Round-robin}.}  
	\label{table:round_robin}
	\resizebox{\linewidth}{!}{
		\begin{tabular}{lllllll}  
		\toprule  
		Number of Visited cells: &5&6&7&8&9&10\\  
		\midrule  
		Sensing Mismatch &2.24&2.40&2.34&1.86&3.01&3.79\\    
		Mission Inefficiency (\%) &25.93&40.42&62.59&75.51&85.69&96.68\\  
		\bottomrule
		\end{tabular}  
	}
\end{table} 

\section{PERFORMANCE EVALUATION}\label{sec:evaluation}
This section assesses the different system parameters and illustrates the results of the three evaluated scenarios: basic synthetic, complex synthetic and transportation scenario. Analytical results about how coordinated mobility influences the sensing performance are illustrated in Appendix~B. 

\subsection{Effect of different parameters}\label{sec:param}

The purpose of the evaluation in this section is to understand how different parameters influence the system performance. 
We use the results of this section to make an empirical choice of these parameters for the rest of the evaluation scenarios. The calculations presented here can be automated for any scenario within a hyper-parameter optimization, which is though not the focus of this paper. 

We use \emph{EPOS-balance} and the basic synthetic scenario to test the mobility range, the number of generated plans, energy utilization, the proportional sensing allocation and the behavior of agents in EPOS.

\begin{figure*}[!htb]
	\centering
	\subfigure[Mobility range.]{
		\includegraphics[height=3.5cm,width=8cm]{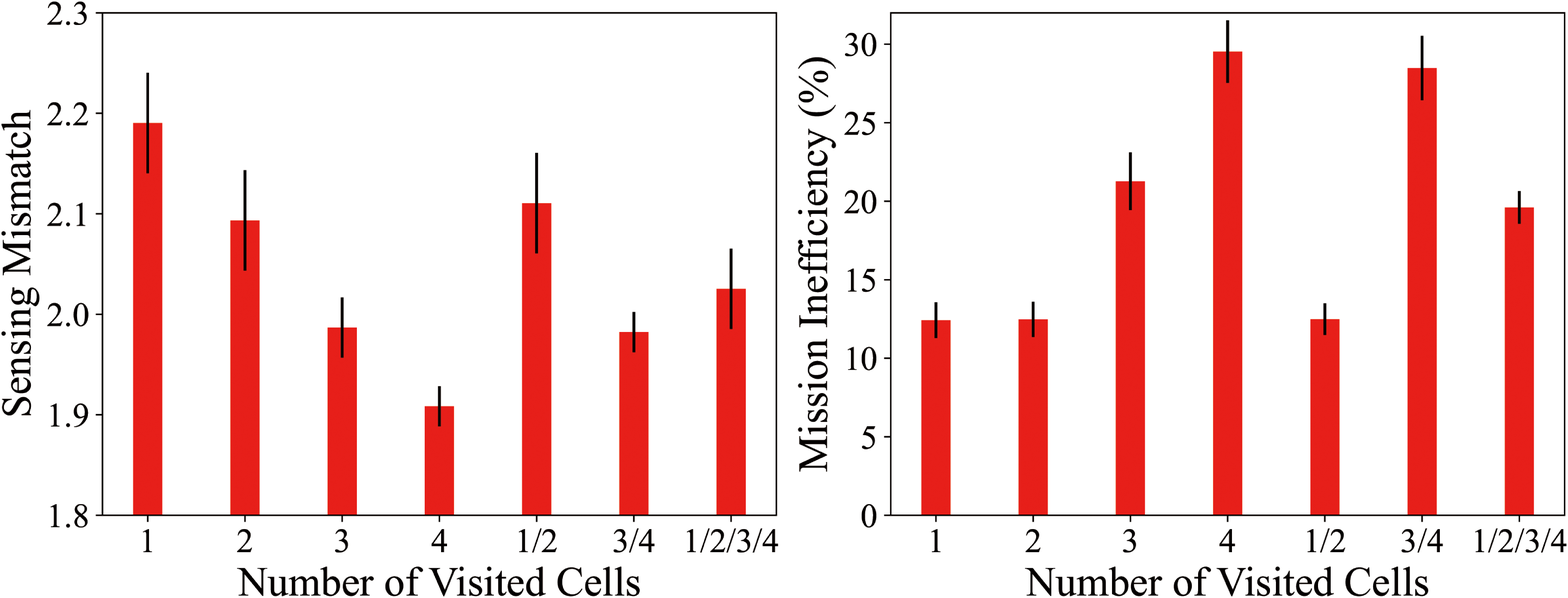}
		\label{fig:visited_cells}
	}
	\subfigure[Number of plans.]{
		\includegraphics[height=3.5cm,width=4.2cm]{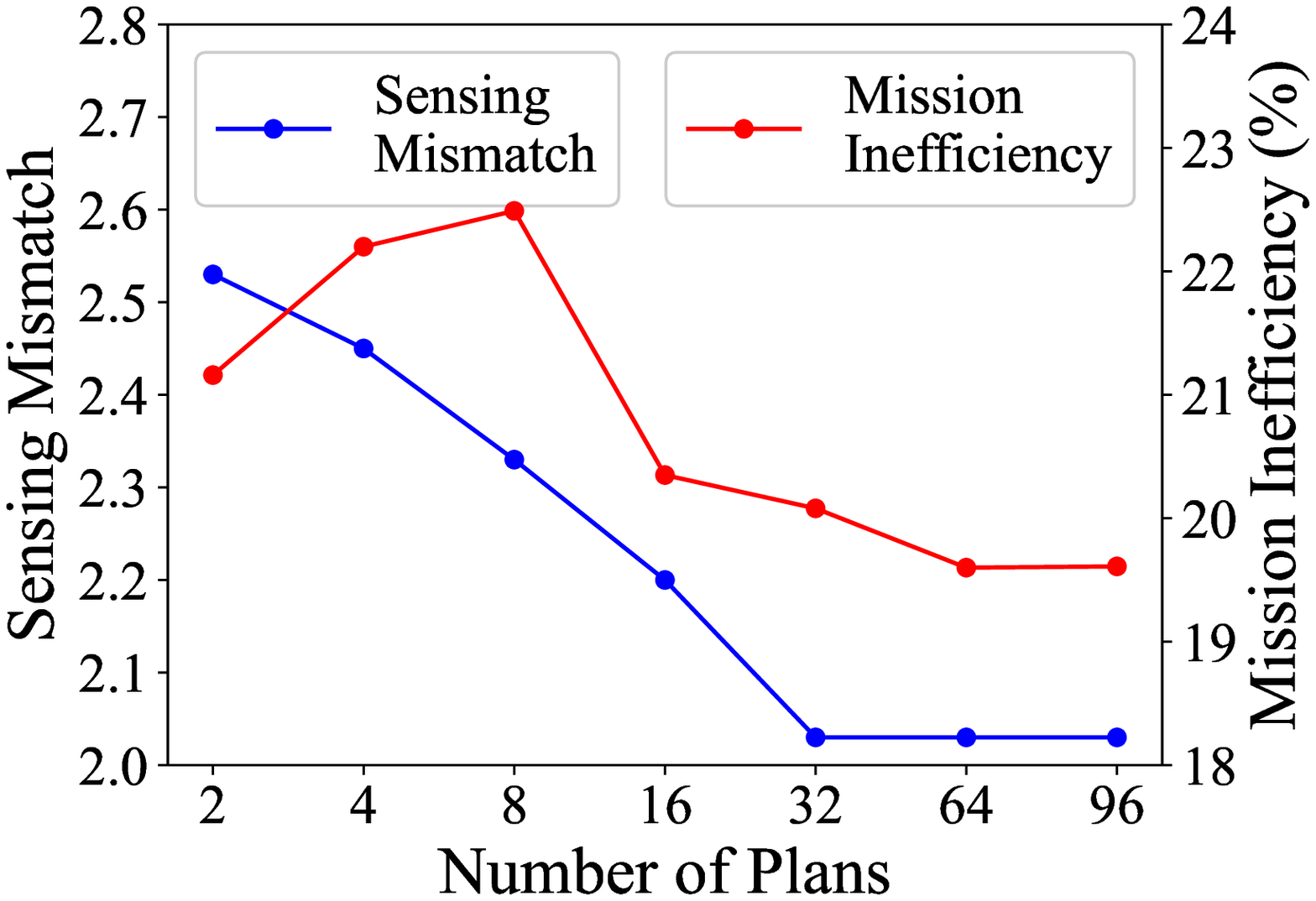}
		\label{fig:plans_num}
	}
     \hspace{3cm}
	\subfigure[Sensing allocation schemes.]{
		\includegraphics[height=3.5cm,width=4.3cm]{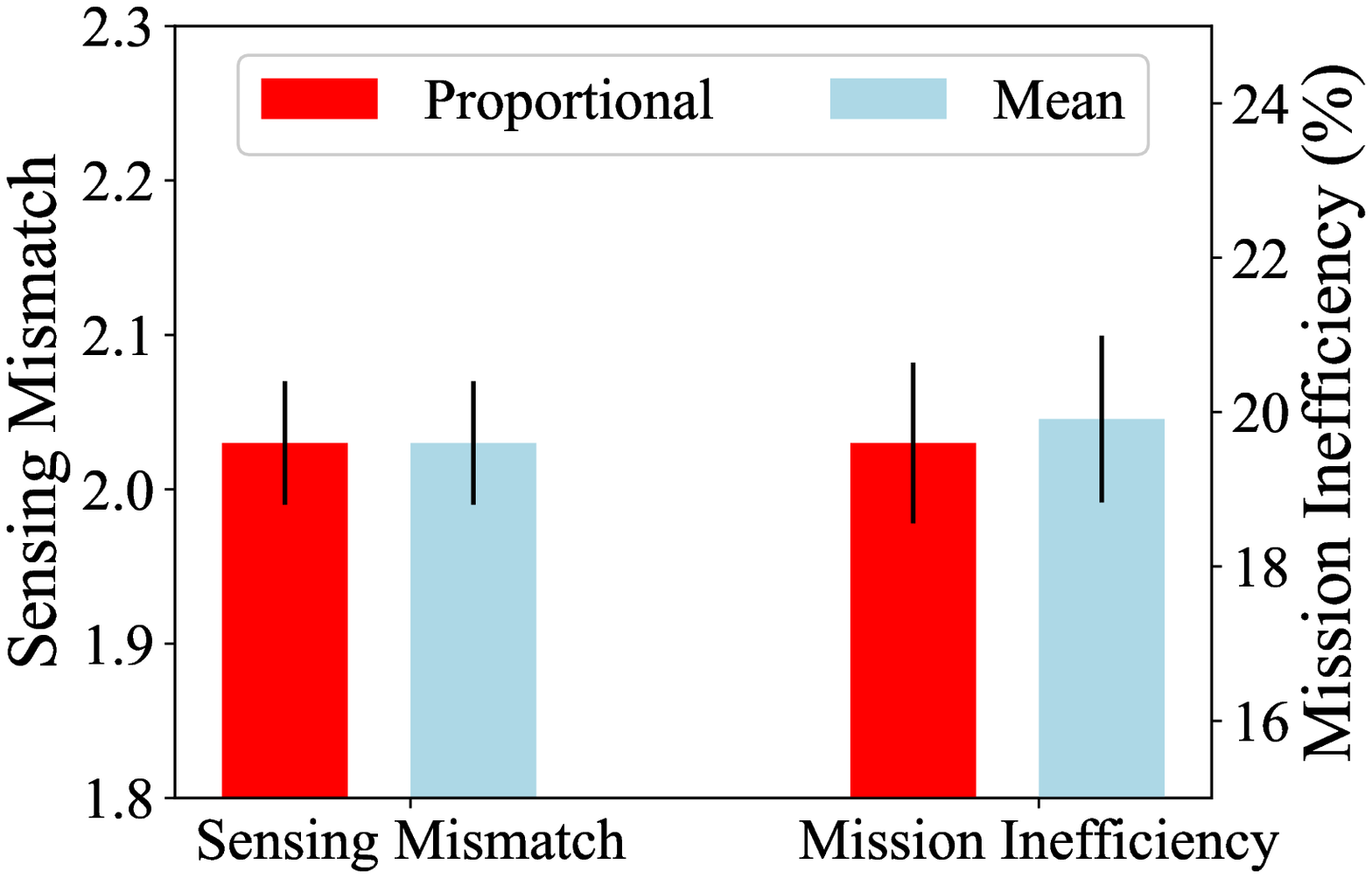}
		\label{fig:no_allocation}
	}
     \subfigure[Energy Utilization.]{
        \includegraphics[height=3.5cm,width=4.3cm]{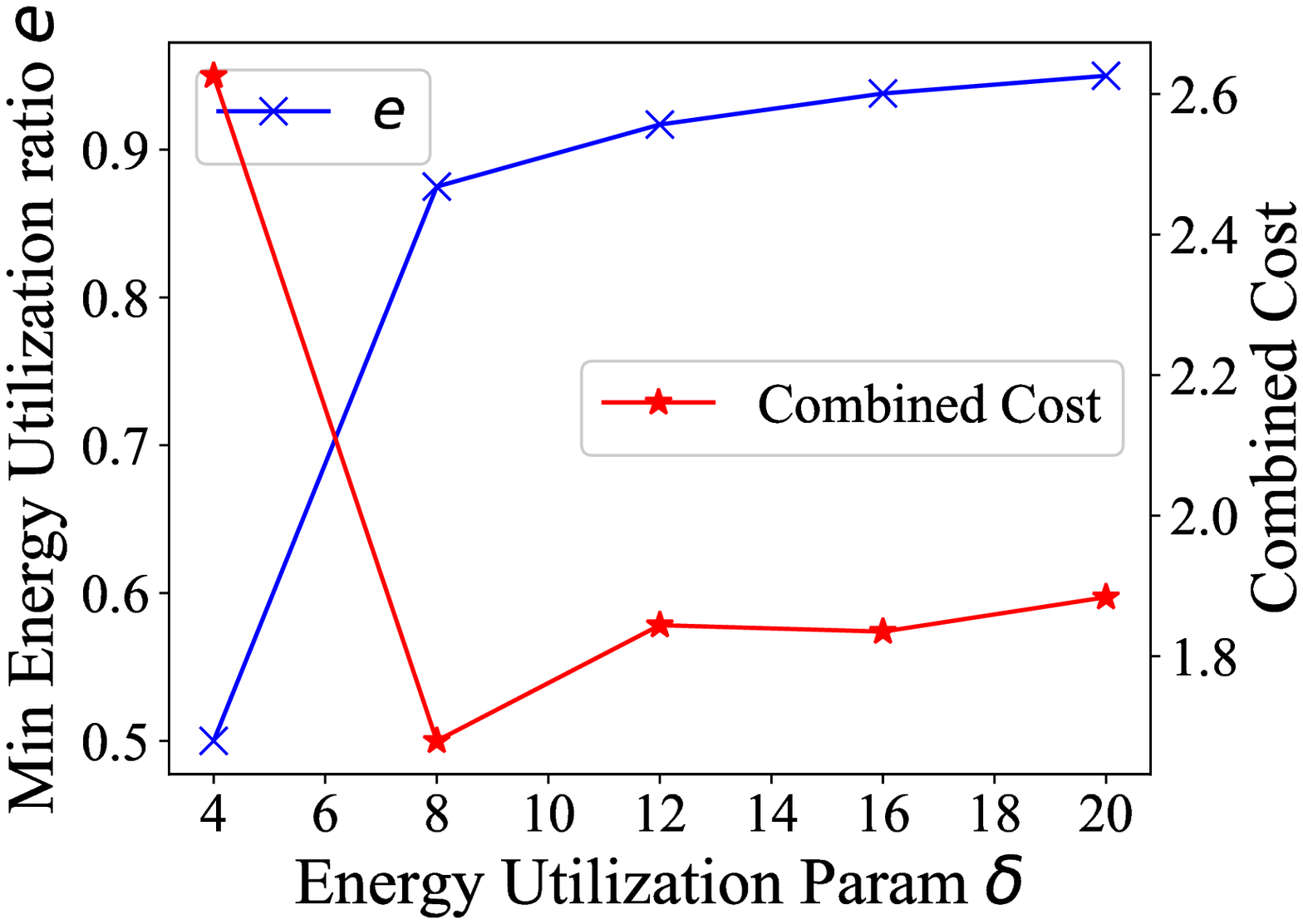}
        \label{fig:utilization}
    }
        \subfigure[Agents' behavior ($\beta$).]{
		\includegraphics[height=3.5cm,width=4.4cm]{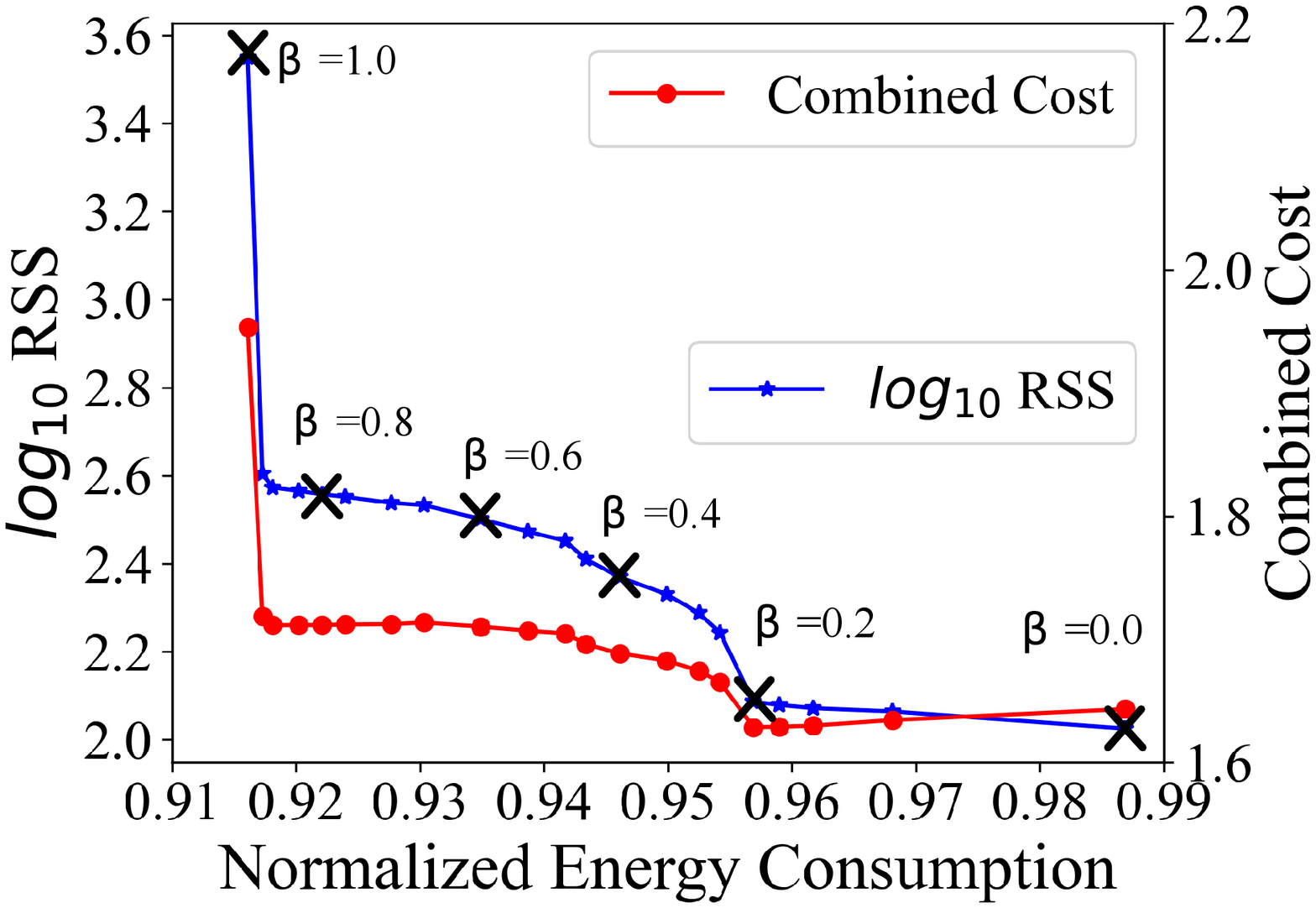}
		\label{fig:preference}
	}
	\caption{Performance comparison for different parameters of the proposed method.}
	\label{fig:cells}
\end{figure*}

\cparagraph{Mobility range}\label{sec:mobility-range}
Three policies of plan generation are compared in EPOS: \emph{balance}, \emph{min sensing mismatch} and \emph{min sensing inefficiency}. A simple test is implemented in the basic synthetic scenario to compare different numbers of visited cells set in the input, see Algorithm 1. As shown in Fig.~\ref{fig:visited_cells}, the higher the number of cells a drone visits, the higher the flying energy is, and the lower the allocated hovering/sensing energy is, which increases the mission inefficiency. To balance the sensing mismatch and mission inefficiency, the policies of min sensing mismatch ($|J_u| = 1|2$, the strategy chooses 1 or 2 cells randomly), min mission inefficiency ($|J_u| = 3|4$), and balance ($|J_u|=1|2|3|4$) are empirically designed. Thus, three policy-based methods for EPOS are proposed that are referred to as \emph{EPOS-mismatch}, \emph{EPOS-inefficiency} and \emph{EPOS-balance} respectively.

\cparagraph{Number of plans}\label{sec:num-of-plans}
See Fig.~\ref{fig:plans_num}, as the number of generated plans increases, both sensing mismatch and mission inefficiency decrease and converge to $64$. Thus, $64$ plans for each agent are generated in EPOS. 

\cparagraph{Sensing allocation: proportional vs. mean}\label{sec:sensing-allocation}
Fig.~\ref{fig:no_allocation} illustrates the performance comparison between proportional and mean sensing allocation defined in Eq.(\ref{eq:sense_allocation}). The mean sensing allocation has higher inefficiency than proportional one. Using the Mann-Whitney U test, the sensing mismatch distribution of the two methods comes with $p=0.08$, the one of mission inefficiency has $p=0.003$. The results demonstrate a statistically higher performance when sensor values collected are proportional, i.e. a similar mission mismatch but a $1.56\%$ higher mission efficiency.

\cparagraph{Energy utilization}\label{sec:energy-utilization}
Fig.~\ref{fig:utilization} illustrates the changes of minimum energy utilization ratio $e$ (when $p = P$ in Eq.(\ref{eq:energy_ratio})) and the combined cost as the parameter $\delta$ increases. The combined cost is the sum of normalized total energy consumption, sensing mismatch and mission inefficiency. According to the results, a $\delta=8$ is selected for the experimental settings.

\cparagraph{Agents' behavior and Pareto optimality}\label{sec:beta-parameter}
Fig.~\ref{fig:preference} illustrates the effect of agents' behavior by varying the parameter $\beta$~\cite{pournaras2018decentralized}. As $\beta$ increases from $0$ to $1$, agents reduce the energy cost of their selected plans, while increasing the sensing mismatch. This is because drones with higher $\beta$ choose a plan with lower energy cost, i.e., the plan with lower energy utilization ratio $e$, and the total sensing values collected is reduced according to Eq.~(\ref{eq:hover_energy}) and~(\ref{eq:total_sense}). To minimize the $log_{10}$ RSS, the value of $\beta=0$ is selected in the proposed methods, while $\beta=0.2$ is the calculated Pareto optimal point referred to as \emph{EPOS-Pareto}, minimizing the combined cost.

\subsection{Evaluation on the basic synthetic scenario}
%This paper compares the performance between the three proposed methods (\emph{EPOS-balance}, \emph{EPOS-mismatch} and \emph{EPOS-inefficiency}) and the baseline methods on the basic and complex synthetic scenario. We compare the results with respect to the total energy consumption, sensing mismatch (RSS) and mission inefficiency. 

\begin{figure*}[!htb]
	\centering
        \subfigure[Total Energy Consumption ($kJ$).]{
		\includegraphics[height=2.9cm,width=5cm]{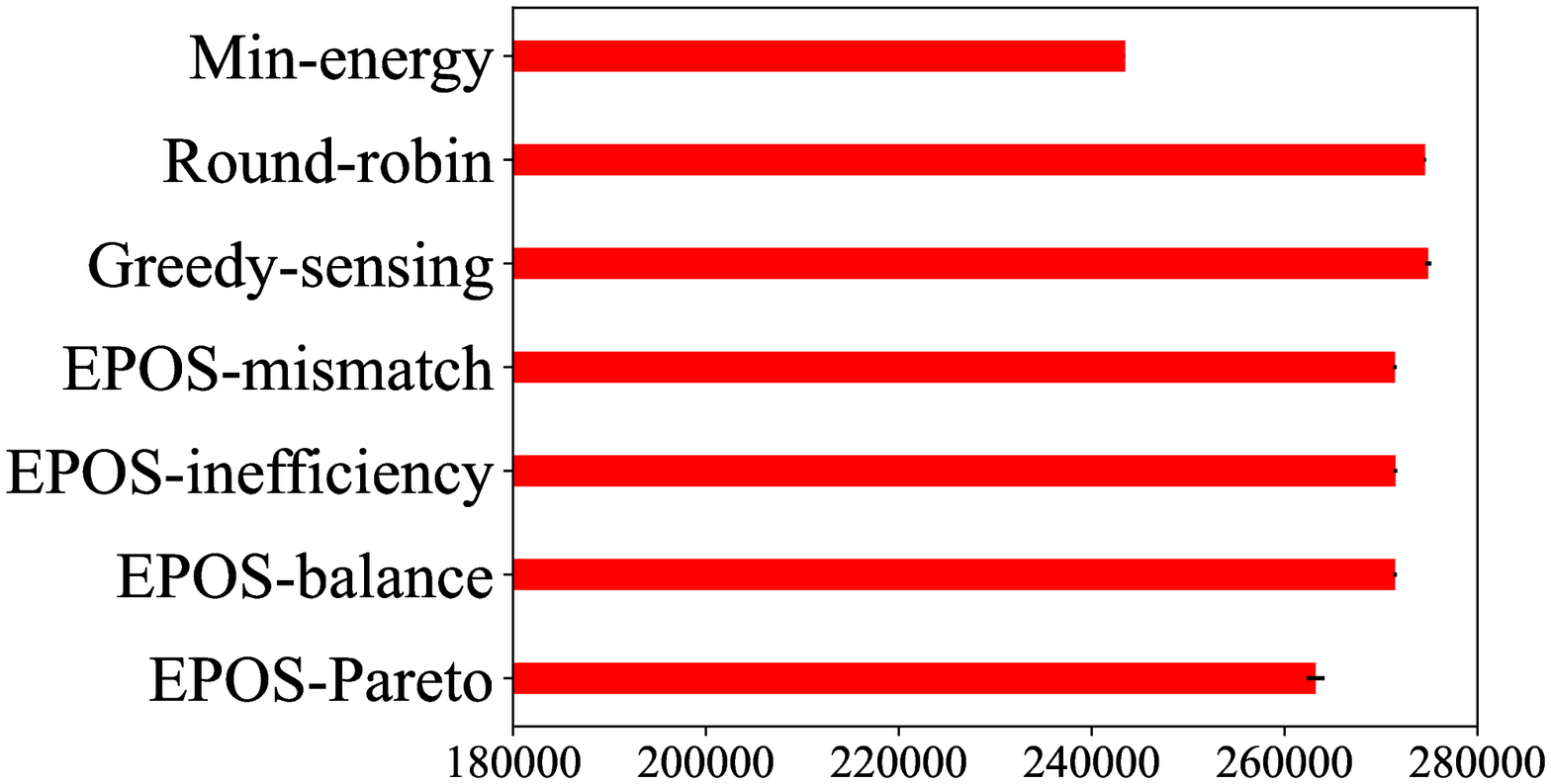}
		\label{fig:basic_energy}
	}
         \subfigure[Sensing Mismatch.]{
		\includegraphics[height=2.9cm,width=3.8cm]{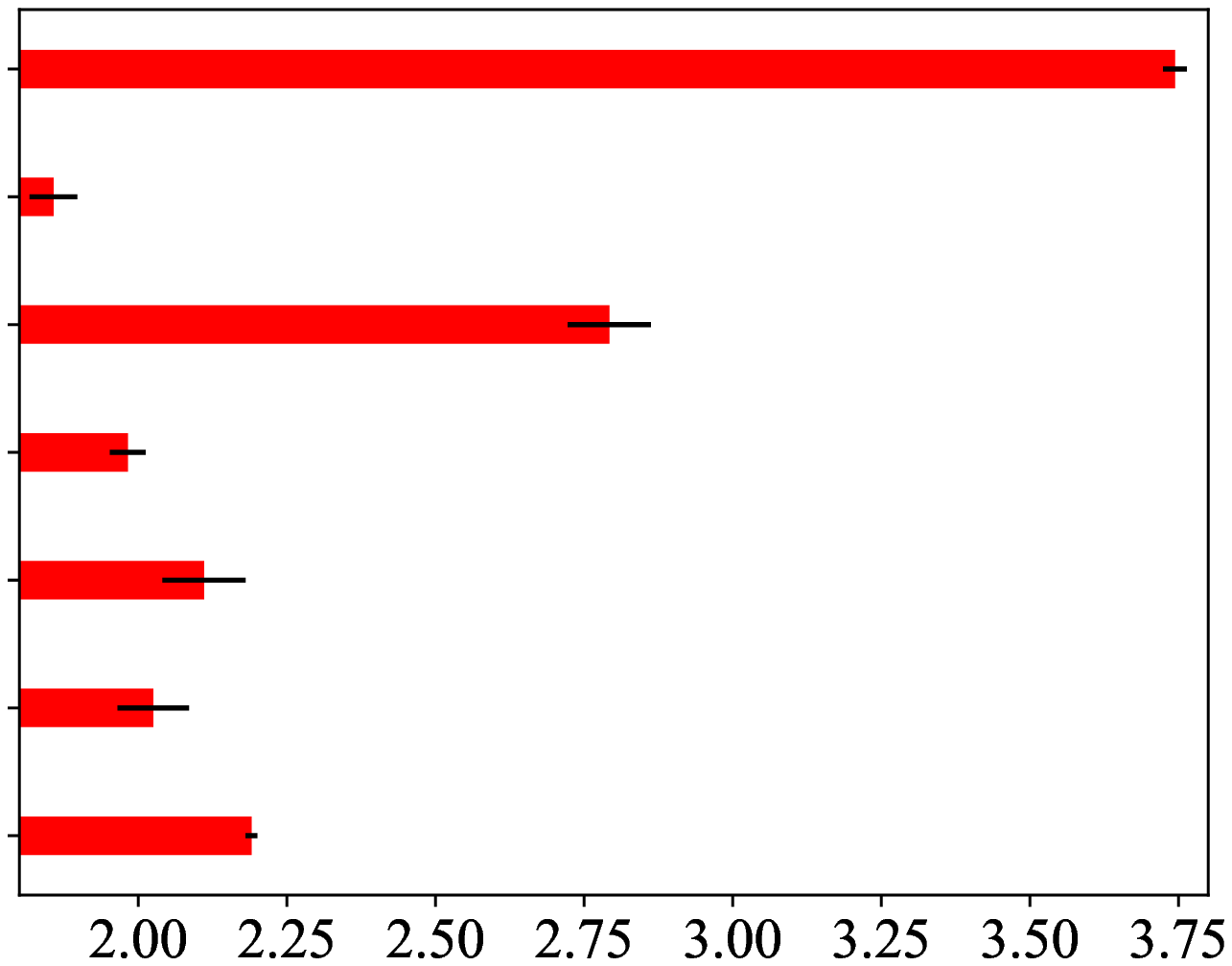}
		\label{fig:basic_global}
	}
         \subfigure[Mission Inefficiency ($\%$).]{
		\includegraphics[height=2.9cm,width=3.8cm]{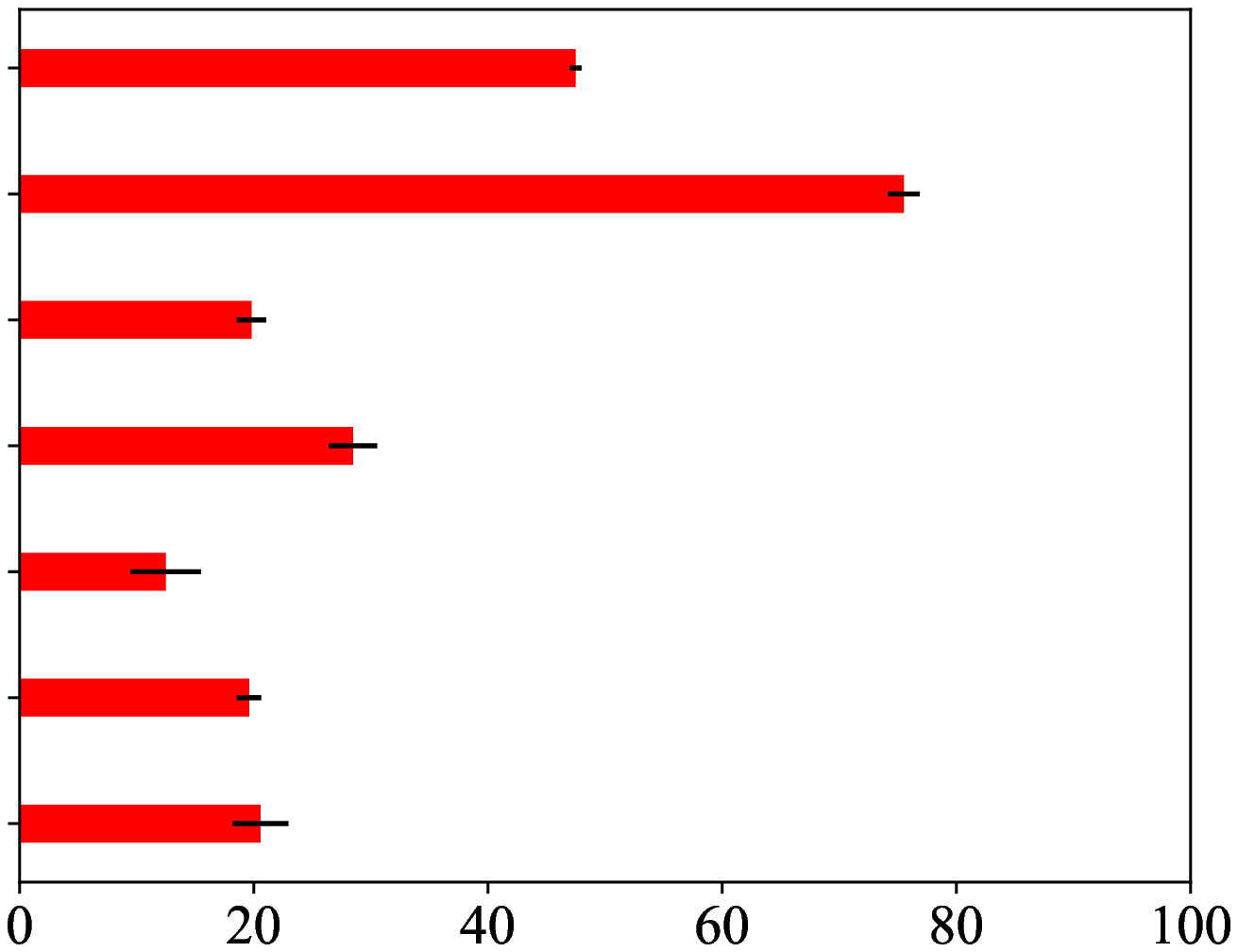}
		\label{fig:basic_mission}
	}
         \subfigure[Combined Cost.]{
		\includegraphics[height=2.9cm,width=3.8cm]{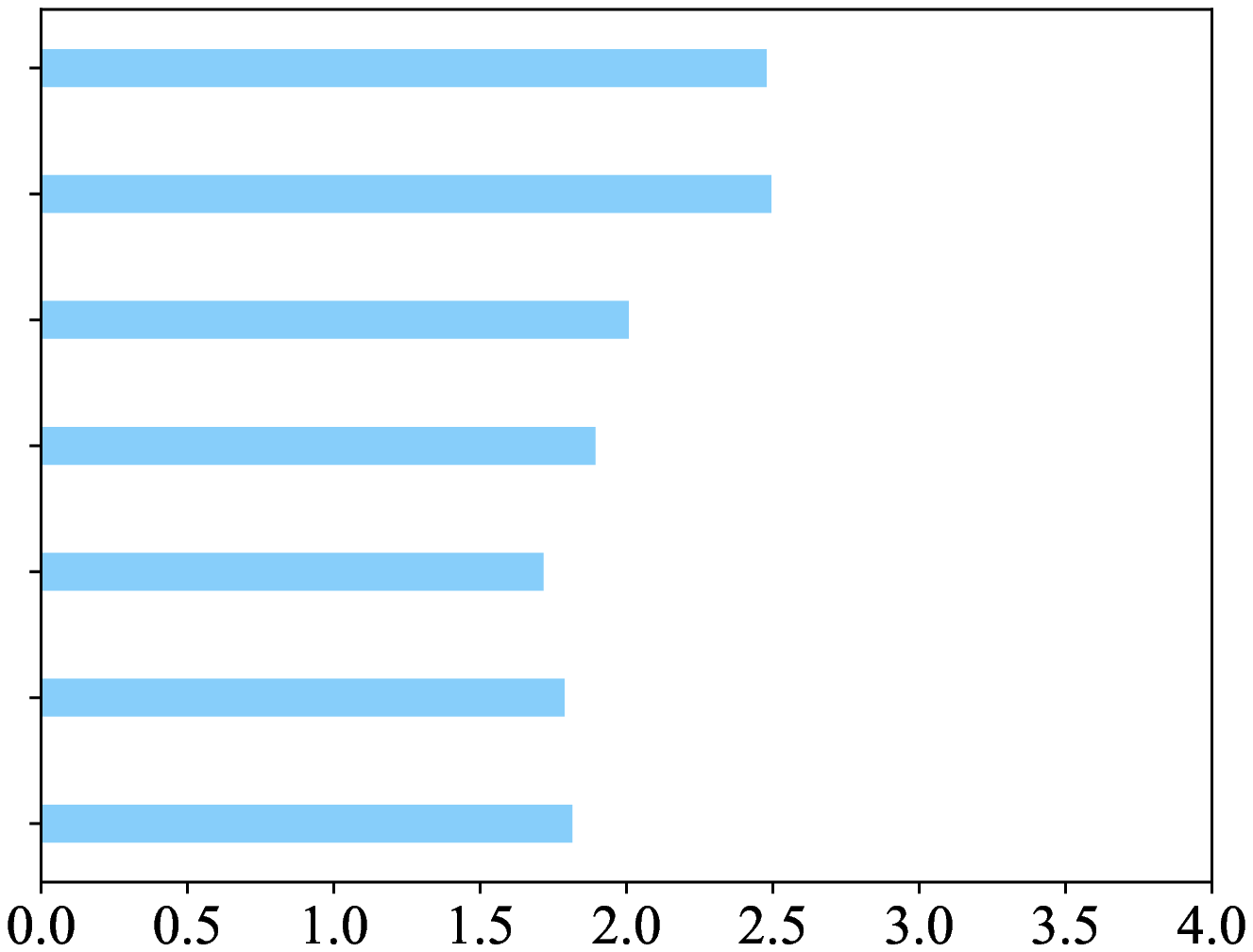}
		\label{fig:average_score}
	}
    \caption{Performance comparison of the six methods on the basic synthetic scenario: 4 base stations, 64 cells and 20000 total target values.}
    \label{fig:basic}
\end{figure*}

As shown in Fig.~\ref{fig:basic}, the three optimization methods based on EPOS coordinate drones to select the plans with the minimum energy utilization ratio $e$ and result in the lowest total energy consumption except \emph{Min-energy} ($\beta=1.0$) and \emph{EPOS-Pareto} ($\beta=0.2$). The coordination of drones achieves the lowest sensing mismatch compared to baselines. Therein, \emph{EPOS-balance} (with a total energy consumption of $271467 kJ$ and sensing mismatch of $2.02$) has the mission inefficiency of $19.6\%$ that is close to \emph{Greedy-sensing}; \emph{EPOS-mismatch} sacrifices mission efficiency to obtain a very low sensing mismatch among three policies, and just $0.12$ higher than \emph{Round-robin}; \emph{EPOS-inefficiency} has the lowest mission inefficiency of $12.49\%$ among all methods. In contrast, without coordination, the \emph{Min-energy} method lowers energy consumption to $233,484 kJ$ at a cost of higher sensing mismatch ($3.74$) and mission inefficiency ($47.5\%$). \emph{Greedy-sensing} and \emph{Round-robin} also come with higher sensing mismatch ($2.79$) and mission inefficiency ($75.51\%$) respectively. 

Fig.~\ref{fig:average_score} illustrates the combined cost calculated from Fig.~\ref{fig:basic_energy}, \ref{fig:basic_global} and \ref{fig:basic_mission}. The proposed method is superior to baseline methods especially when combining all performance metrics.

\subsection{Evaluation on the complex synthetic scenario}
There are four dimensions we study here: (1) the number of dispatched drones used in the sensing mission, (2) the total target values to collect, (3) the number of cells into which the same map is split, and (4) the number of base stations. Fig.~\ref{fig:heatmap} shows the performance comparison between \emph{EPOS-balance}, \emph{Greedy-sensing} and \emph{Round-robin} methods. 

\begin{figure}[!htb]
	\centering
	\subfigure[Total Energy Consumption (Base stations: up=4, down=16).]{
		\includegraphics[width=0.48\textwidth]{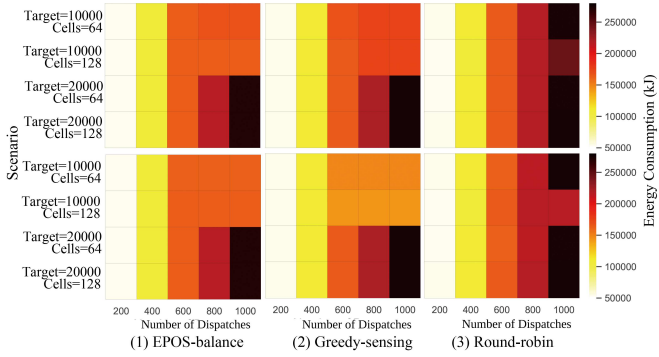}
		\label{fig:heatmap_energy}
	}
	\subfigure[Sensing Mismatch (Base stations: up=4, down=16).]{
		\includegraphics[width=0.48\textwidth]{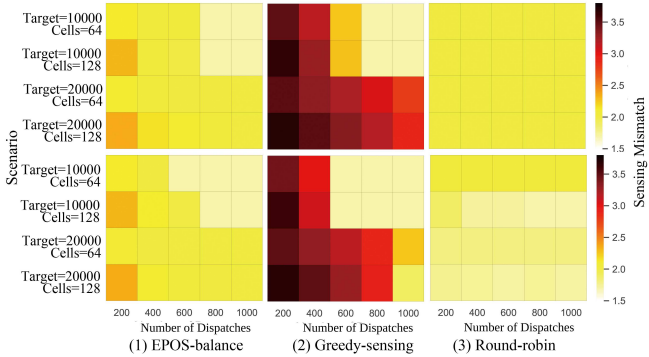}
		\label{fig:heatmap_global}
	}
	\subfigure[Mission Inefficiency (Base stations: up=4, down=16).]{
		\includegraphics[width=0.48\textwidth]{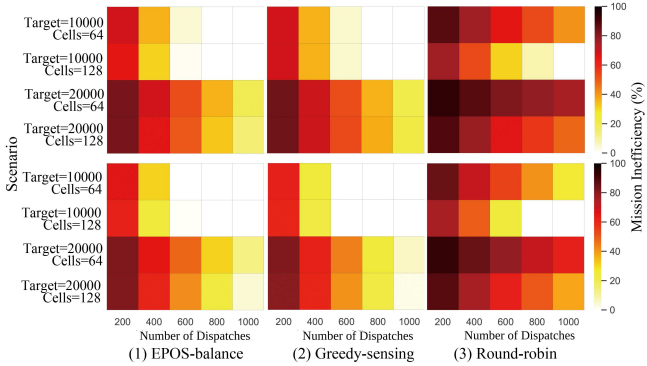}
		\label{fig:heatmap_mission}
	}
	\caption{Performance comparison under varying parameters: total target values (10000 and 20000), the number of cells (64 and 128), the number of base stations (heatmap triplets: up=4, down=16), and the number of drone dispatches (from 200 to 1000).}
	\label{fig:heatmap}
\end{figure}

\cparagraph{Number of dispatches}
Take an example of the complex synthetic scenario (target=$20000$, cells=$64$, base stations=$4$). As the number of dispatches increases, the total energy consumption of \emph{EPOS-balance} rises, while mission inefficiency decreases proportionally (from $271467kJ$, $83\%$ to $54268kJ$, $19\%$ in Fig.~\ref{fig:heatmap_energy} and Fig.~\ref{fig:heatmap_mission}). This shows that drones take full advantage of energy resources to collect sensor values efficiently. Compared to \emph{Round-robin}, \emph{EPOS-balance} collects more data (exceeds the average of $38.5\%$) with a lower number of dispatches and energy. Furthermore, as the number of dispatches decreases, \emph{EPOS-balance} shows a lower increase of sensing mismatch (increased by $4.21\%$) than \emph{Greedy-sensing} (increased by $26.17\%$). If some drones fail to be dispatched due to attacks or other factors, \emph{EPOS-balance} is proved to effectively mitigate the penalties of over-sensing and under-sensing, which validates the resilience of the proposed method.

\cparagraph{Total target values}
If the total target values to collect decrease from $20000$ to $10000$, the mission inefficiency of \emph{EPOS-balance} reduces to $0$ and the energy consumption keeps constant when the number of dispatches exceeds $500$. This is because the total target values represent the total hovering time of drones, and \emph{EPOS-balance} does not waste energy resources when drones complete their sensing tasks. In Fig.~\ref{fig:heatmap_global}, the sensing mismatch of \emph{Greedy-sensing} is highly sensitive to the change of target values. It increases by $1.09$, while \emph{EPOS-balance} increases only by $0.34$.

\cparagraph{Number of cells}
If the number of cells increases from $64$ to $128$, the size of each cell decreases, and the distance between any two cells decreases, which reduces the travel distance of drones. As a consequence, drones spend more energy on collecting sensor values according to Eq.(\ref{eq:hover_energy}) and Eq.(\ref{eq:total_sense}), which reduces the mission inefficiency. In Fig.~\ref{fig:heatmap_mission}, \emph{EPOS-balance} shows a higher decrease in mission inefficiency (decreased by $4\%$) than \emph{Greedy-sensing} (decreased by $1\%$) as the number of cells increases. Furthermore, \emph{EPOS-balance} collects higher quality data with a lower number of cells and drones as shown in Fig.~\ref{fig:heatmap_global}.

\cparagraph{Number of base stations}
If the number of base stations increases from $4$ to $16$, the travel distances of drones reduce and the total sensor values collected rise. As a result, mismatch and inefficiency of \emph{EPOS-balance} decrease. \emph{EPOS-balance} shows an overall superior performance as targets and base stations vary. 

In summary, \emph{EPOS-balance} has the lowest sensing mismatch (close to \emph{Round-robin}) and mission inefficiency (close to \emph{Greedy-sensing}) as the variables in the environment change. This confirms that \emph{EPOS-balance} has the highest performance in overall under the same battery constraints and different sensing scenarios.

\subsection{Traffic monitoring by coordinated drones} \label{sec:real}
This paper shows the significant and broad impact of the proposed coordinated sensing method on a transportation scenario, where a swarm of drones perform traffic monitoring. The evaluation uses real-world data from vehicle trajectories collected by a swarm of drones in the congested downtown area of Athens, Greece~\cite{barmpounakis2020new}. This application scenario envisions a large-scale and long-term monitoring of traffic congestion using coordinated drones to measure and predict traffic patterns. As shown in Fig.~\ref{fig:traffic_map}, a swarm of $10$ drones hover over $10$ cells for $20$ time periods to record traffic streams of $6$ types of vehicles: \emph{car}, \emph{taxi}, \emph{bus}, \emph{medium vehicle}, \emph{heavy vehicle} and \emph{motorcycle}. Each drone departs from and returns to one of the two base stations and hovers up to $25$ minutes at a single cell.

\begin{figure}
	\centering
	\includegraphics[scale=0.4]{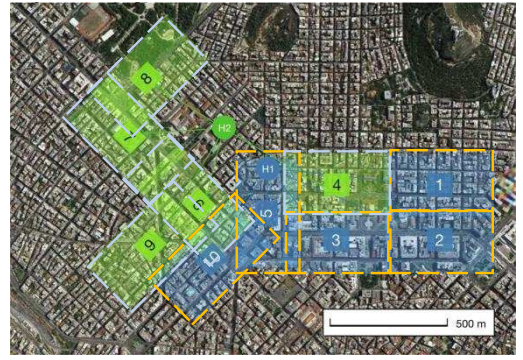}
	\caption{A swarm of $10$ drones hovering over the central business district of Athens over five days to record traffic flows in a congested area of a $1.3 km^2$ with more than $100$ km lanes of road network, around $100$ busy intersections, more than $50$ bus stops and close to half a million trajectories. There are $5$ drones that depart from cell $H1$ and travel to the cell $1$, $2$, $3$, $5$ and $10$ (shown in blue), whereas the other $5$ drones set off from $H2$ and hover over the rest of cells (shown in green)~\cite{barmpounakis2020new}.}
	\label{fig:traffic_map}
\end{figure}

\begin{figure*}[!htb]
	\centering
	\subfigure[Spatio-temporal distribution of vehicles.]{
		\includegraphics[height=5.6cm, width=\textwidth]{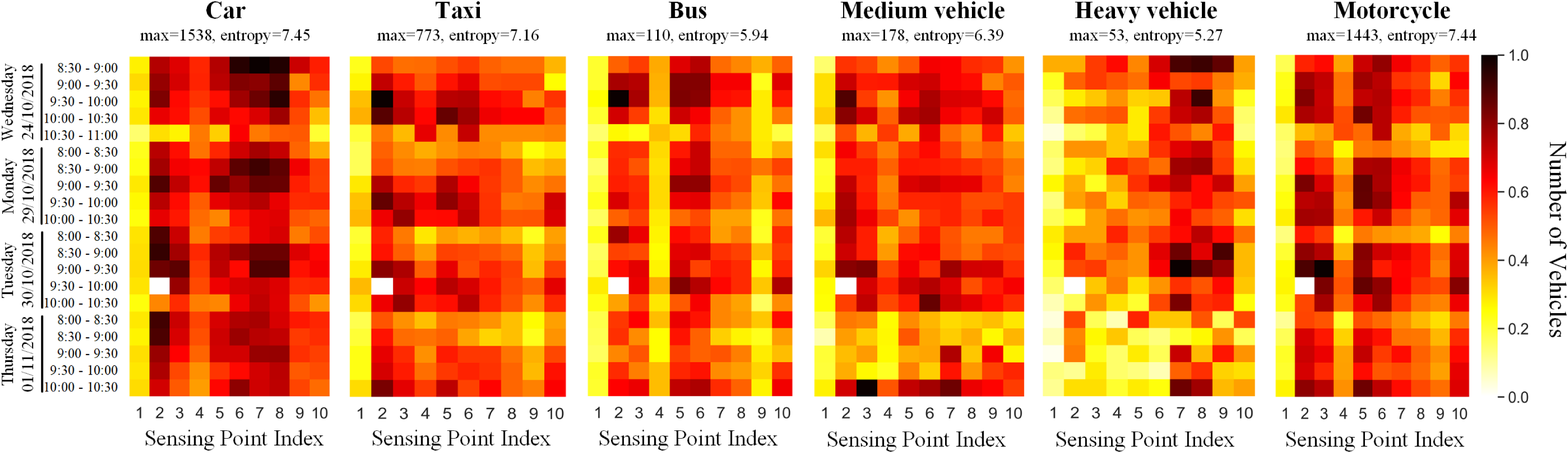}
		\label{fig:transport_dist}
	}
	\subfigure[Accuracy (correlation) of vehicles detection.]{
		\includegraphics[height=2.9cm, width=\textwidth]{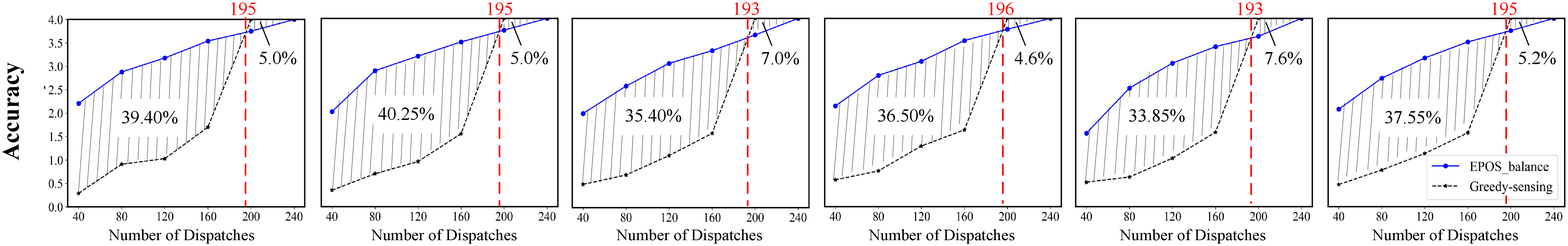}
		\label{fig:transport_global}
	}
	\subfigure[Efficiency of vehicles detection.]{
		\includegraphics[height=2.9cm, width=\textwidth]{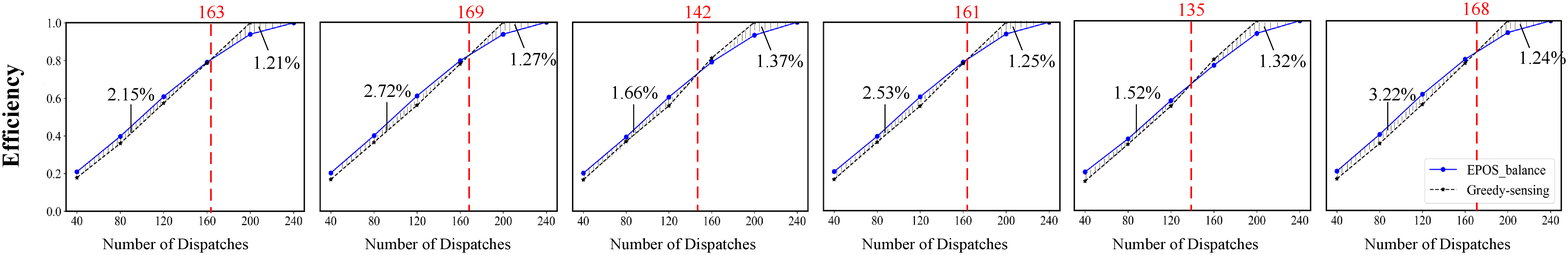}
		\label{fig:transport_mission}
	}
	\caption{The comparison results among six types of vehicles (car, taxi, bus, medium vehicle, heavy vehicle, and motorcycle) in a downtown area of Athens based on the open traffic monitoring data collected by a drone swarm.}
	\label{fig:transport}
\end{figure*}

% The required sensing tasks (or targets) to monitor traffic vehicles is encoded by a matrix of size $M \times N$: $((25,25,25,...)(25,25,25,...)...)$, $M=20$, $N=10$, which divides the tasks within one day into $20$ equal time periods and $10$ cells. The value in the vector denotes that drones are required to hover over the cell $n$ for $25$ minutes during the period $m$, $m =1,2,...,20$. During the task coordination, each drone chooses only one period and several cells based on the proposed method. For instance, a drone monitoring cell $1$ and cell $2$ at period $1$ has the plan of $((8,8,0,...)(0,0,0,...)...)$, where the hovering time of $8$ is calculated according to Eq.(\ref{eq:sense_allocation}).
\cparagraph{Sensing model}
We model the application scenario of coordinated drones that perform traffic sensing. To improve the sensing quality, we calculate the required sensing tasks (or targets) to be proportional to the spatio-temporal normalized distribution of vehicles shown in Fig.~\ref{fig:transport_dist}. We assume that typical distributions can be derived from historical data~\cite{pNEUMA_dataset}. The higher the likelihood a vehicle type drives through a cell, the higher the target value is, and $T_n \leq (25 \times 20)$, $n = {1,2,...,10}$. Over $20$ time periods (30 min), a certain number of drones in each period is dispatched.

\cparagraph{Metrics and Baselines}
The performance metrics are adjusted to evaluate more directly the observability of vehicles in this traffic monitoring scenario: \emph{accuracy} and \emph{efficiency}. Accuracy $A$ denotes the correlation between the distributions of observed vehicles and actual existing ones. Efficiency $E$ denotes the ratio of actual vehicles captured by the drones over the total ones. They are formulated as follows:
\begin{align}
    & A = \log_{10} (\frac{1}{\sum_{m=1}^{M}(V^*_m - V_m)^2}),\\
    & E = \frac{\sum_{m=1}^{M} V^*_m}{\sum_{m=1}^{M} V_m},
\end{align}
\noindent where $V_m$ indicates the total number of vehicles at time unit $m$ acquired from pNEUMA~\cite{barmpounakis2020new} (we set one minute as a time unit), and $V^*_m$ denotes the number of vehicles monitored by drones at time unit $m$ using the proposed method, $V^*_m = \sum_{n \in \mathcal{N}} h^u_{m,n} \cdot V_m$. Moreover, the \emph{Greedy-sensing} method is used as the baseline. It requires a low energy consumption and emulates the vehicles data collection in pNEUMA~\cite{barmpounakis2020new}.

\cparagraph{Experimental results}
Fig.\ref{fig:transport} illustrates the performance comparison between the proposed \emph{EPOS-balance} and the baseline for varying numbers of dispatched drones used in traffic monitoring. The results in Fig.~\ref{fig:transport_global} show that the accuracy of \emph{EPOS-balance} among six types of vehicles is significantly higher and more stable than that of \emph{Greedy-sensing} when the number of dispatches is lower than $200$. This is because with the total of $200$ drones' dispatches under \emph{Greedy-sensing}, each cell within each period is monitored by exactly one drone. With a lower number of drones' dispatches (i.e., scarce resources), however, \emph{EPOS-balance} coordinates drones to monitor all cells over all time periods, preventing over-sensing and under-sensing. Fig.~\ref{fig:transport_mission} illustrates that \emph{EPOS-balance} is relatively more efficient than \emph{Greedy-sensing} with no more than $160$ dispatches. This is because \emph{EPOS-balance} coordinates drones to collect sensor values that are proportional to the distributions of vehicles, which increases the number of vehicles observed by drones. Note that there is a strong linear relationship between the threshold for the number of dispatches (see Fig.~\ref{fig:transport_mission}) and the entropy of vehicles distribution (see Fig.~\ref{fig:transport_dist}), with a Pearson coefficient correlation of $0.93$ and corresponding p-value of $0.007$.

In summary, for lower than $160$ dispatches, \emph{EPOS-balance} is approximately $46.45\%$ more accurate and $2.88\%$ more efficient than \emph{Greedy-sensing} among six vehicle types monitored. This verifies the remarkable performance of the proposed method under a scarce number of drone resources, requiring less than $80\%$ of the drones to achieve equivalent or higher performance than \emph{Greedy-sensing}.

\section{CONCLUSION AND FUTURE WORK}\label{sec:conclusion}
In conclusion, this paper illustrates a decentralized coordination model to solve the task self-assignment problem for scalable, resilient and flexible spatio-temporal sensing by a swarm of drones. To ensure energy-aware and efficient coordination of sensing, this paper introduces a novel plan generation strategy with three policies. Extensive experiments demonstrate that the proposed method is adaptive to complex sensing scenarios and has better sensing performance than existing methods. The evaluation on real-world data shows that the designed model has a strong potential in optimizing traffic monitoring in the Smart Cities using a limited number of drones.

Nevertheless, the proposed method can be further improved towards several research avenues: (1) Use of real-world datasets in other applications of Smart Cities, including last-mile delivery, disaster response, and smart farming. (2) Use of other learning method, such as multi-agent reinforcement learning algorithm, to extend the study of coordinated drones with charging capabilities and obstacle/collision avoidance. (3) Use of efficient wireless communication technology and special hardware to implement the coordination capability on board and online.

\section*{ACKNOWLEDGEMENTS}

This research is supported by a UKRI Future Leaders Fellowship (MR-/W009560/1): \emph{Digitally Assisted Collective Governance of Smart City Commons– ARTIO}, the White Rose Collaboration Fund: \emph{Socially Responsible AI for Distributed Autonomous Systems} and a 2021 Alan Turing Fellowship. Thanks to Emmanouil Barmpounakis and Nikolas Geroliminis for their support in using the pNEUMA dataset. 

\section*{REFERENCES}

\def\refname{\vadjust{\vspace*{-1em}}} %Please don't do this in a real paper.
\bibliographystyle{unsrt}
\bibliography{DataCollection}

\begin{thebibliography}{10}

\bibitem{wu2016addsen}
Di~Wu, Dmitri~I Arkhipov, Minyoung Kim, Carolyn~L Talcott, Amelia~C Regan,
  Julie~A McCann, and Nalini Venkatasubramanian.
\newblock Addsen: Adaptive data processing and dissemination for drone swarms
  in urban sensing.
\newblock {\em IEEE Transactions On Computers}, 66(2):183--198, 2016.

\bibitem{inoue2020satellite}
Yoshio Inoue.
\newblock Satellite-and drone-based remote sensing of crops and soils for smart
  farming--a review.
\newblock {\em Soil Science and Plant Nutrition}, 66(6):798--810, 2020.

\bibitem{butilua2022urban}
Eugen~Valentin Butil{\u{a}} and R{\u{a}}zvan~Gabriel Boboc.
\newblock Urban traffic monitoring and analysis using unmanned aerial vehicles
  ({UAVs}): A systematic literature review.
\newblock {\em Remote Sensing}, 14(3):620, 2022.

\bibitem{barmpounakis2020new}
Emmanouil Barmpounakis and Nikolas Geroliminis.
\newblock On the new era of urban traffic monitoring with massive drone data:
  The pneuma large-scale field experiment.
\newblock {\em Transportation research part C: emerging technologies},
  111:50--71, 2020.

\bibitem{poudel2022task}
Sabitri Poudel and Sangman Moh.
\newblock Task assignment algorithms for unmanned aerial vehicle networks: A
  comprehensive survey.
\newblock {\em Vehicular Communications}, page 100469, 2022.

\bibitem{fu2019secure}
Zhangjie Fu, Yuanhang Mao, Daojing He, Jingnan Yu, and Guowu Xie.
\newblock Secure multi-{UAV} collaborative task allocation.
\newblock {\em IEEE Access}, 7:35579--35587, 2019.

\bibitem{yanmaz2018drone}
Ev{\c{s}}en Yanmaz, Saeed Yahyanejad, Bernhard Rinner, Hermann Hellwagner, and
  Christian Bettstetter.
\newblock Drone networks: Communications, coordination, and sensing.
\newblock {\em Ad Hoc Networks}, 68:1--15, 2018.

\bibitem{zhou2020uav}
Yongkun Zhou, Bin Rao, and Wei Wang.
\newblock {UAV} swarm intelligence: Recent advances and future trends.
\newblock {\em IEEE Access}, 8:183856--183878, 2020.

\bibitem{alighanbari2005decentralized}
Mehdi Alighanbari and Jonathan~P How.
\newblock Decentralized task assignment for unmanned aerial vehicles.
\newblock In {\em Proceedings of the 44th IEEE Conference on Decision and
  Control}, pages 5668--5673. IEEE, 2005.

\bibitem{mac2018development}
Thi~Thoa Mac, Cosmin Copot, Robin De~Keyser, and Clara~M Ionescu.
\newblock The development of an autonomous navigation system with optimal
  control of an {UAV} in partly unknown indoor environment.
\newblock {\em Mechatronics}, 49:187--196, 2018.

\bibitem{bupe2015relief}
Paul Bupe, Rami Haddad, and Fernando Rios-Gutierrez.
\newblock Relief and emergency communication network based on an autonomous
  decentralized {UAV} clustering network.
\newblock In {\em SoutheastCon 2015}, pages 1--8. IEEE, 2015.

\bibitem{nik2021using}
Vahid~M Nik and Amin Moazami.
\newblock Using collective intelligence to enhance demand flexibility and
  climate resilience in urban areas.
\newblock {\em Applied Energy}, 281:116106, 2021.

\bibitem{pournaras2018decentralized}
Evangelos Pournaras, Peter Pilgerstorfer, and Thomas Asikis.
\newblock Decentralized collective learning for self-managed sharing economies.
\newblock {\em ACM Transactions on Autonomous and Adaptive Systems (TAAS)},
  13(2):1--33, 2018.

\bibitem{pournaras2020collective}
Evangelos Pournaras.
\newblock Collective learning: A 10-year odyssey to human-centered distributed
  intelligence.
\newblock In {\em 2020 IEEE International Conference on Autonomic Computing and
  Self-Organizing Systems (ACSOS)}, pages 205--214. IEEE, 2020.

\bibitem{Pournaras2020}
Evangelos Pournaras, Srivatsan Yadhunathan, and Ada Diaconescu.
\newblock Holarchic structures for decentralized deep learning: a performance
  analysis.
\newblock {\em Cluster Computing}, 23(1):219--240, 2020.

\bibitem{qin20223}
Chuhao Qin, Fethi Candan, Lyudmila Mihaylova, and Evangelos Pournaras.
\newblock 3, 2, 1, drones go! a testbed to take off {UAV} swarm intelligence
  for distributed sensing.
\newblock In {\em Proceedings of the 2022 UK Workshop on Computational
  Intelligence}. Springer Nature, 2022.

\bibitem{Fanitabasi2020}
Farzam Fanitabasi, Edward Gaere, and Evangelos Pournaras.
\newblock A self-integration testbed for decentralized socio-technical systems.
\newblock {\em Future Generation Computer Systems}, 113:541--555, 2020.

\bibitem{Stolaroff2018}
Joshuah~K. Stolaroff, Constantine Samaras, Emma~R. O'Neill, Alia Lubers,
  Alexandra~S. Mitchell, and Daniel Ceperley.
\newblock Energy use and life cycle greenhouse gas emissions of drones for
  commercial package delivery.
\newblock {\em Nature Communications}, 9(1), feb 2018.

\bibitem{qin_data2022}
Chuhao Qin and Evangelos Pournaras.
\newblock {EPOS-based Plans for Drones}.
\newblock 12 2022.
\newblock doi = {"https://doi.org/10.6084/m9.figshare.21432804.v10"}.

\bibitem{gao2018multi}
Yang Gao, Yingzhou Zhang, Shurong Zhu, and Yi~Sun.
\newblock Multi-{UAV} task allocation based on improved algorithm of
  multi-objective particle swarm optimization.
\newblock In {\em 2018 International Conference on Cyber-Enabled Distributed
  Computing and Knowledge Discovery (CyberC)}, pages 443--4437. IEEE, 2018.

\bibitem{wu2021multi}
Xueli Wu, Yanan Yin, Lei Xu, Xiaojing Wu, Fanhua Meng, and Ran Zhen.
\newblock Multi-{UAV} task allocation based on improved genetic algorithm.
\newblock {\em IEEE Access}, 9:100369--100379, 2021.

\bibitem{chen2018multi}
Yongbo Chen, Di~Yang, and Jianqiao Yu.
\newblock Multi-{UAV} task assignment with parameter and time-sensitive
  uncertainties using modified two-part wolf pack search algorithm.
\newblock {\em IEEE Transactions on Aerospace and Electronic Systems},
  54(6):2853--2872, 2018.

\bibitem{motlagh2019energy}
Naser~Hossein Motlagh, Miloud Bagaa, and Tarik Taleb.
\newblock Energy and delay aware task assignment mechanism for {UAV}-based
  {IoT} platform.
\newblock {\em IEEE Internet of Things Journal}, 6(4):6523--6536, 2019.

\bibitem{zhou2018mobile}
Zhenyu Zhou, Junhao Feng, Bo~Gu, Bo~Ai, Shahid Mumtaz, Jonathan Rodriguez, and
  Mohsen Guizani.
\newblock When mobile crowd sensing meets {UAV}: Energy-efficient task
  assignment and route planning.
\newblock {\em IEEE Transactions on Communications}, 66(11):5526--5538, 2018.

\bibitem{bartolini2020multi}
Novella Bartolini, Andrea Coletta, Gaia Maselli, et~al.
\newblock A multi-trip task assignment for early target inspection in squads of
  aerial drones.
\newblock {\em IEEE Transactions on Mobile Computing}, 20(11):3099--3116, 2020.

\bibitem{bartolini2019task}
Novella Bartolini, Andrea Coletta, and Gaia Maselli.
\newblock On task assignment for early target inspection in squads of aerial
  drones.
\newblock In {\em 2019 IEEE 39th International Conference on Distributed
  Computing Systems (ICDCS)}, pages 2123--2133. IEEE, 2019.

\bibitem{alighanbari2006robust}
Mehdi Alighanbari and Jonathan How.
\newblock Robust decentralized task assignment for cooperative {UAV}s.
\newblock In {\em AIAA Guidance, Navigation, and Control Conference and
  Exhibit}, page 6454, 2006.

\bibitem{chen2022consensus}
Jie Chen, Xianguo Qing, Fang Ye, Kai Xiao, Kai You, and Qian Sun.
\newblock Consensus-based bundle algorithm with local replanning for
  heterogeneous multi-{UAV} system in the time-sensitive and dynamic
  environment.
\newblock {\em The Journal of Supercomputing}, 78(2):1712--1740, 2022.

\bibitem{elloumi2018monitoring}
Mouna Elloumi, Riadh Dhaou, Benoit Escrig, Hanen Idoudi, and Leila~Azouz
  Saidane.
\newblock Monitoring road traffic with a uav-based system.
\newblock In {\em 2018 IEEE Wireless Communications and Networking Conference
  (WCNC)}, pages 1--6. IEEE, 2018.

\bibitem{omidshafiei2017decentralized}
Shayegan Omidshafiei, Ali-Akbar Agha-Mohammadi, Christopher Amato, Shih-Yuan
  Liu, Jonathan~P How, and John Vian.
\newblock Decentralized control of multi-robot partially observable markov
  decision processes using belief space macro-actions.
\newblock {\em The International Journal of Robotics Research}, 36(2):231--258,
  2017.

\bibitem{fanitabasi2020self}
Farzam Fanitabasi, Edward Gaere, and Evangelos Pournaras.
\newblock A self-integration testbed for decentralized socio-technical systems.
\newblock {\em Future Generation Computer Systems}, 113:541--555, 2020.

\bibitem{gerostathopoulos2019trapped}
Ilias Gerostathopoulos and Evangelos Pournaras.
\newblock Trapped in traffic? a self-adaptive framework for decentralized
  traffic optimization.
\newblock In {\em 2019 IEEE/ACM 14th International Symposium on Software
  Engineering for Adaptive and Self-Managing Systems (SEAMS)}, pages 32--38.
  IEEE, 2019.

\bibitem{outay2020applications}
Fatma Outay, Hanan~Abdullah Mengash, and Muhammad Adnan.
\newblock Applications of unmanned aerial vehicle (uav) in road safety, traffic
  and highway infrastructure management: Recent advances and challenges.
\newblock {\em Transportation research part A: policy and practice},
  141:116--129, 2020.

\bibitem{castells2020cycling}
David Castells-Graells, Christopher Salahub, and Evangelos Pournaras.
\newblock On cycling risk and discomfort: urban safety mapping and bike route
  recommendations.
\newblock {\em Computing}, 102(5):1259--1274, 2020.

\bibitem{allen2016variance}
Zeyuan Allen-Zhu and Elad Hazan.
\newblock Variance reduction for faster non-convex optimization.
\newblock In {\em International conference on machine learning}, pages
  699--707. PMLR, 2016.

\bibitem{kizilatecs2013nearest}
G{\"o}zde Kizilate{\c{s}} and Fidan Nuriyeva.
\newblock On the nearest neighbor algorithms for the traveling salesman
  problem.
\newblock In {\em Advances in Computational Science, Engineering and
  Information Technology}, pages 111--118. Springer, 2013.

\bibitem{Phantom_UAV}
Company {DJI}.
\newblock {DJI} {Phantom} 4 {Pro}.
\newblock \url{https://www.dji.com/uk/phantom-4-pro/info#specs}.
\newblock Accessed: 2022-11-16.

\bibitem{wierzbicki2018multi}
Damian Wierzbicki.
\newblock Multi-camera imaging system for {UAV} photogrammetry.
\newblock {\em Sensors}, 18(8):2433, 2018.

\bibitem{Nikolic2019}
Jovan Nikolic and Evangelos Pournaras.
\newblock Structural self-adaptation for decentralized pervasive intelligence.
\newblock In {\em 2019 22nd Euromicro Conference on Digital System Design
  (DSD)}, pages 562--571. IEEE, 2019.

\bibitem{alwateer2019two}
Majed Alwateer and Seng~W Loke.
\newblock A two-layered task servicing model for drone services: Overview and
  preliminary results.
\newblock In {\em 2019 IEEE International Conference on Pervasive Computing and
  Communications Workshops (PerCom Workshops)}, pages 387--390. IEEE, 2019.

\bibitem{pNEUMA_dataset}
E.~Barmpounakis and N.~Geroliminis.
\newblock {pNEUMA} dataset.
\newblock \url{https://open-traffic.epfl.ch/}.
\newblock Accessed: 2022-11-16.

\end{thebibliography}

\begin{appendices}
    \section{Power consumption model}
Drones spend energy to surpass gravity force and counter drag forces due to wind and forward motions. A drone controls the speed of each rotor to achieve the thrust $T$ and pitch $\theta$ necessary to stay aloft and travel forward at the desired velocity while balancing the weight and drag forces. For a drone with mass $m_{b}$ and its battery with mass $m_{c}$, we define the total required thrust as follows:
\begin{equation}
	\mathcal{T} = (m_{b} + m_{c}) \cdot g + F_{d},
\end{equation}
\noindent where $g$ is the gravitational constant, and $F_{d}$ is the drag force that depends on air speed and air density. For steady flying, the drag force can be calculated by the pitch angle $\theta$ as:
\begin{equation}
	F_{d} = (m_{b} + m_{c}) \cdot g \cdot tan(\theta).
=\end{equation}

Building on the model in~\cite{Stolaroff2018}, the power consumption with forward velocity and forward pitch is given by:
\begin{equation}
	P^\mathsf{f} = (v \cdot sin\theta + v_{i}) \cdot \frac{\mathcal{T}}{\epsilon},
	\label{power_flight}
\end{equation}
\noindent where $v$ is the average ground speed; $\epsilon$ is the overall power efficiency of the drone;  $v_{i}$ is the induced velocity required for given $T$ and can be found by solving the nonlinear equation:
\begin{equation}
	v_{i} = \frac{2 \cdot \mathcal{T}}{\pi \cdot d^{2} \cdot r \cdot \rho \cdot \sqrt{(v \cdot cos\theta)^{2} + (v \cdot sin\theta + v_{i})^{2}}},
\end{equation} 
\noindent where $d$ and $r$ are the diameter and number of drone rotors; $\rho$ is the density of air. 

Moreover, the power consumption for hovering of a drone is calculated by:
\begin{equation}
	P^\mathsf{h} = \frac{\mathcal{T}^{3/2}}{\epsilon \cdot \sqrt{\frac{1}{2} \pi \cdot d^{2} \cdot n \cdot \rho}}.
	\label{power_hover}
\end{equation} 

\section{Mobility and Sensing Quality: Analytical Results}

Two theorems are introduced that link the performance metrics of mission inefficiency and sensing mismatch with the mobility of the coordinated drones, in particular with their flying coverage modelled by the number of visiting cells. 

\begin{theorem}
In a mission of $\mathcal{U}$ drones, starting from their base stations, flying with a constant ground speed over an area consisting of $N$ cells to collect sensor data, and returning back consuming all energy of their battery, the mission inefficiency is proportional to the number of random visited cells $|J_u|$:
\begin{equation}
        \alpha_1 \cdot |J_u| \rightarrow f_1(J_u) = 1 - \frac{ \sum_{u \in \mathcal{U}} S(J_u) }{\sum_{n \in \mathcal{N}} T_n}, 
        \label{eq:effi_theo}
\end{equation}

\noindent if the drones collect sensor values over the visited cells proportionally to required target values, where $\alpha_1$ is a positive constant and $f_1(J_u)$ is the function of mission inefficiency according to the optimization objective of Eq.(\ref{eq:model_optimal}).\label{theorem}
\end{theorem}

\begin{proof}

    Based on Eq.~(\ref{eq:hover_energy}) and~(\ref{eq:total_sense}) and since each drone uses all its battery capacity for flying and hovering, the number of collected sensor values over the cells $J_u$ is: 
    
    \begin{equation}
        S(J_u) = \sum_{n \in \mathcal{N}} S_{u, n} = \frac{C_u \cdot e - P^\mathsf{f}_u \cdot \tau(J_u)}{P^\mathsf{h}_u \cdot f}.
        \label{eq:sense_appen1}
    \end{equation}
    
    \noindent Eq.(\ref{eq:model_optimal}) can be reformulated as:

    \begin{equation}
        1 - \frac{ \sum_{u \in \mathcal{U}} \sum_{n \in \mathcal{N}} S_{u,n} }{\sum_{n \in \mathcal{N}} T_n}=1 - \frac{ \sum_{u \in \mathcal{U}} S(J_u) }{\sum_{n \in \mathcal{N}} T_n}\coloneqq f_1(J_u) ,\label{eq:inefficiency-appen}
    \end{equation}

    \noindent given that $\sum_{u \in \mathcal{U}} \sum_{n \in \mathcal{N}} S_{u,n} = \sum_{u \in \mathcal{U}} S(J_u)$. The total flying time $\tau(J_u)$ of a drone $u$ can be modelled as: 

    \begin{equation}
        \tau(J_u) \approx (|J_u|-1) \cdot \tilde{\tau} + 2 \cdot  \tilde{\tau}_{B},
        \label{eq:flight_appen}
    \end{equation}

    \noindent where $\tilde{\tau}$ is the mean expected traveling time between any two random cells and $\tilde{\tau}_{B}$ is the mean expected traveling time between the base station of the drone and a random cell. Assume the number of random visited cells increases from $|J_u|$ to $|J_u'|$, where $|J_u| < |J_u'|$. Then, the total flying time without hovering is also likely to increase as $\tau(J_u) < \tau(J_u')$ given that each drone $u$ flies with a constant ground speed. The influence $f_1(J_u') - f_1(J_u)$ on the mission inefficiency based on Eq.(\ref{eq:inefficiency-appen}) is calculated as follows: 
    
    \begin{equation}
        \begin{split}
            f_1(J_u') - f_1(J_u) 
                &= \frac{\sum_{u \in \mathcal{U}} (\overbrace{S(J_u)}^{Eq.(\ref{eq:sense_appen1})} - \overbrace{S(J_u')}^{Eq.(\ref{eq:sense_appen1})})}{\sum_{n \in \mathcal{N}} T_n} \\
                &= \frac{\sum_{u \in \mathcal{U}} \frac{P^\mathsf{f}_u}{P^\mathsf{h}_u \cdot f}}{\sum_{n \in \mathcal{N}} T_n} \cdot (\overbrace{\tau(J_u')}^{Eq.(\ref{eq:flight_appen})} - \overbrace{\tau(J_u)}^{Eq.(\ref{eq:flight_appen})})\\
                &\approx \frac{\sum_{u \in \mathcal{U}} \frac{P^\mathsf{f}_u}{P^\mathsf{h}_u \cdot f}}{\sum_{n \in \mathcal{N}} T_n} \cdot \tilde{\tau} \cdot (|J_u'| - |J_u|)\\
                & \leftarrow \alpha_1 \cdot (|J_u'| - |J_u|).
        \end{split}
        \label{eq:effi_prove}
    \end{equation}
    Since all parameters in $\frac{\sum_{u \in \mathcal{U}} \frac{P^\mathsf{f}_u}{P^\mathsf{h}_u \cdot f}}{\sum_{n \in \mathcal{N}} T_n} \cdot \tilde{\tau}$ are positive, $\alpha_1 > 0$, and therefore the mission inefficiency is proportional to the number of random visited cells.
\end{proof}

\begin{theorem}
In a mission of $\mathcal{U}$ drones, starting from their base stations, flying with a constant ground speed over an area consisting of $N$ cells to collect sensor data, and returning back consuming all energy of their battery, the sensing mismatch is inverse proportional to the number of random visited cells $|J_u|$:
    \begin{equation}
        \alpha_2 \cdot |J_u| \rightarrow f_2(J_u) = \sum_{n \in \mathcal{N}} (T_n - \sum_{u \in \mathcal{U}} S_{u,n})^2, 
        \label{eq:rss_theo}
    \end{equation}
    if and only if $|J_u'|+|J_u| < N$ when increasing the number of visited cells from $|J_u|$ to $|J_u'|$ and if the drones collect sensor values over the visited cells proportionally to required target values, where $\alpha_2$ is a negative constant and $f_2(J_u)$ is a function of sensing mismatch according to the optimization objective of Eq. (\ref{eq:model_optimal}).\label{theorem2} 
\end{theorem}

\begin{proof}

\begin{figure}[!htb]
    \centering
    \includegraphics[scale=0.25]{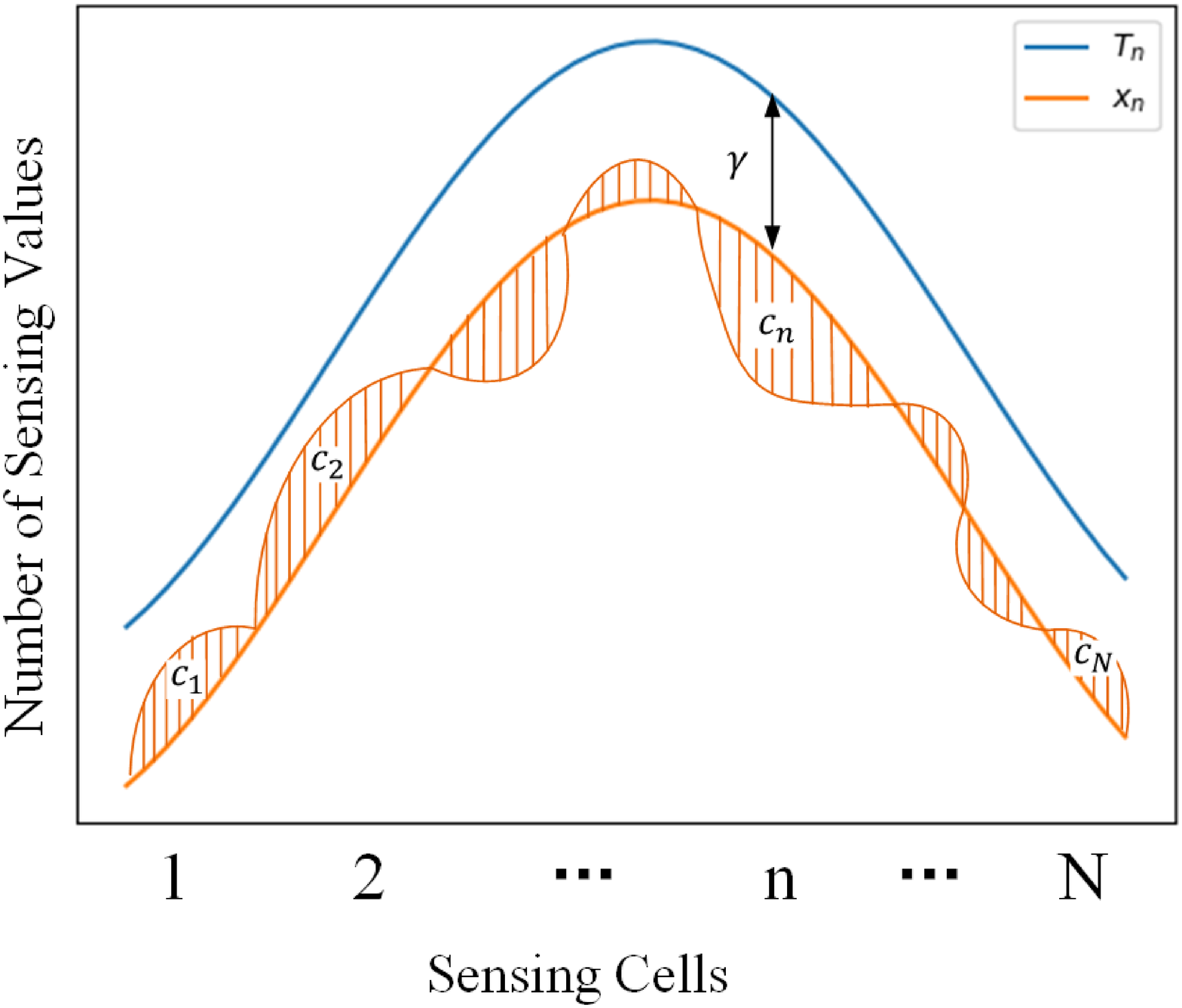}
    \caption{An illustration to assist the proof of Theorem~\ref{theorem2}: The distribution of targets and total sensing values collected by drones.}
    \label{fig:prove}
\end{figure}

    Figure~\ref{fig:prove} assists this proof. Since sensor values over the different cells are collected proportionally to the required target values $T_n$, the collected sensor values are modelled by $x_n + v_n = \sum_{u \in \mathcal{U}} S_{u,n}$, $\forall n \in \{1,...,N\}$, and $x_n\leq T_n$. Each $x_n$ corresponds to the collected sensor values with an optimal matching to the required target values (min RSS), while $v_n \in \{ +c_n, -c_n, 0 \}$ models mismatches such that $\sum_{n \in \mathcal{N}} (x_n + v_n) \approx \sum_{n \in \mathcal{N}} x_n$, and thus $\sum_{n \in \mathcal{N}} v_n \approx 0$. Moreover, the mission inefficiency in the optimal collected sensor data $x_n$ is given by $\gamma$ such that $\gamma \coloneqq  T_n - x_n $. The optimal matching between $T_n$ and $x_n$ denotes that $\gamma \geq 0$ is constant $\forall n \in \{1,...,N\}$. Therefore, it holds that: 

    \begin{equation}
    \begin{split}
            \sum_{n \in \mathcal{N}} (T_n - \sum_{u \in \mathcal{U}} S_{u,n}) &= \sum_{n \in \mathcal{N}} (T_n -(x_n + v_n)) \\
                &\approx \sum_{n \in \mathcal{N}} (T_n - x_n) \\
                &\approx N \cdot \gamma > 0
        \end{split}
        %\sum_{n \in \mathcal{N}} (T_n - \sum_{u \in \mathcal{U}} S_{u,n}) = \sum_{n \in \mathcal{N}} (T_n -(x_n + v_n)) \approx \sum_{n \in \mathcal{N}} (T_n - x_n) \approx N \cdot C > 0.
        \label{eq:inefficiency-mismatch}
    \end{equation}

    Eq.(\ref{eq:inefficiency-mismatch}) can be squared to calculate the sensing mismatch $f_2(J_u)$ as follows:
    
    \begin{equation}
        \begin{split}
            f_2(J_u) &= \sum_{n \in \mathcal{N}} (T_n - (x_n + v_n))^2 \\
                &= \sum_{n \in \mathcal{N}} (\gamma - v_n)^2 \\
                % &= \sum_{n \in \mathcal{N}} (T_n - x_n)^2 - 2 \sum_{n \in \mathcal{N}} v_n (T_n-x_n) +\sum_{n \in \mathcal{N}} v_n^2 \\
                &= N \cdot \gamma^2 - 2 \gamma \cdot \sum_{n \in \mathcal{N}} v_n + \sum_{n \in \mathcal{N}} v_n^2 \\
                &= N \cdot \gamma^2 + \sum_{n \in \mathcal{N}} c_n^2.
        \end{split}
        \label{eq:rss}
    \end{equation}
    
    \noindent The higher the $c_n$ is, the higher the $f_2(J_u)$.

    The distribution of the mission inefficiency values $T_n - (x_n + v_n)$ among $N$ cells is determined by the selection of the cells by each drone. By assuming that each of the $U$ drones chooses the visiting cells randomly (with replacement), the distribution among $N$ cells can be explained by a Binomial distribution:
    
    \begin{equation}
        P(X=k, J_u) = \binom{U}{k} \cdot p(J_u)^k \cdot (1-p(J_u))^{U-k},
        \label{eq:binom}
    \end{equation}

    \noindent where $p(J_u)$ is the probability that a drone $u$ chooses $|J_u|$ number of cells from the total of $N$ cells that do not contain the cell $n$. This results in mismatch at cell $n$ that either originates from (i) an under-sensing $v_n = -c_n$, if a drone $u$ has a high probability $p(J_u)$ to miss cell $n$ from $J_u$, or (ii) an over-sensing $T_n - (x_n + v_n) = C + c_n$ if this probability is low (see Fig.~\ref{fig:prove}). The probability $p(J_u)$ can be formulated as: 
   
    \begin{equation}
        p(J_u) = \binom{N-1}{|J_u|} / \binom{N}{|J_u|} = 1 - \frac{|J_u|}{N},
        \label{eq:binom_prob}
    \end{equation}
     
    \noindent which expresses that the higher the number of visiting cells is, the lower the probability of drone $u$ to miss a cell $n$. 

    The mismatch $c_n$ at cell $n$ for a drone visiting $J_u$ points can be modeled by a Binomial distribution: 

    \begin{equation}
        c_n(J_u) := k_n (J_u) \cdot \overline{S}_n(J_u),
        \label{eq:dispersion_a}
    \end{equation}
    
    \noindent with the expected values of $k_n(J_u)$ denoting the average number of drones that miss cell $n$ and $\overline{S}_n(J_u)$ denoting the average number of collected values by each drone $u$. The expressions of these values are formulated as follows: 
    
    \begin{align}
        & k_n(J_u) = U \cdot p(J_u), \label{eq:expected_a} \\
        & \overline{S}_n(J_u) = \sum_{u \in \mathcal{U}} S(J_u) \cdot \frac{T_n}{U \cdot \sum_{n \in \mathcal{N}} T_n}, \label{eq:averageS_a} 
    \end{align}
    
    By increasing the number of sensing cells from $|J_u|$ to $|J_u'|$, the influence $f_2(J_u') - f_2(J_u)$ on the sensing mismatch can be formulated using Eq.~(\ref{eq:rss}) as follows: 
    
    \begin{equation}
        \begin{split}
            f_2(J_u') - f_2(J_u) 
                &= \sum_{n \in \mathcal{N}} [\overbrace{c_n(J_u')^2}^{Eq.(\ref{eq:dispersion_a},\ref{eq:expected_a},\ref{eq:averageS_a})} - \overbrace{c_n(J_u)^2}^{Eq.(\ref{eq:dispersion_a},\ref{eq:expected_a},\ref{eq:averageS_a})}] \\
                &= \sum_{n \in \mathcal{N}} \frac{T_n}{\sum_{n \in \mathcal{N}} T_n} \cdot \overbrace{p(J_u')^2}^{Eq.(\ref{eq:binom_prob})} \cdot \overbrace{\sum_{u \in \mathcal{U}} S(J_u')^2}^{Eq.(\ref{eq:effi_prove})} \\ 
                &\quad - \sum_{n \in \mathcal{N}} \frac{T_n}{\sum_{n \in \mathcal{N}} T_n} \cdot \overbrace{p(J_u)^2}^{Eq.(\ref{eq:binom_prob})} \cdot \overbrace{\sum_{u \in \mathcal{U}} S(J_u)^2}^{Eq.(\ref{eq:effi_prove})} \\
                &\leftarrow \sum_{n \in \mathcal{N}} T_n \cdot [\frac{\alpha_1}{N} \cdot |J_u'|^2 - \alpha_1 \cdot |J_u'|] \\
                &\quad - \sum_{n \in \mathcal{N}} T_n \cdot [\frac{\alpha_1}{N} \cdot |J_u|^2 - \alpha_1 \cdot |J_u|] \\
                &\leftarrow [\frac{|J_u'|+|J_u|}{N} - 1] \cdot \alpha_1 \cdot \sum_{n \in \mathcal{N}} T_n \cdot (|J_u'| - |J_u|) \\
                &\leftarrow \alpha_2 \cdot (|J_u'| - |J_u|)
        \end{split}
        \label{eq:rss_prove}
    \end{equation}

    \noindent Thus, it holds that $\alpha_2 = [\frac{|J_u'|+|J_u|}{N} - 1] \cdot \alpha_1 \cdot \sum_{n \in \mathcal{N}} T_n < 0$ if and only if $|J_u'|+|J_u| < N$, where in this case the sensing mismatch is proven to be reverse proportional to the number of random visited cells $J_u$. 
\end{proof}

\section{Results of Experimental Settings}

Figure~\ref{fig:statistics} illustrates the required number of tested maps to achieve a stable mean global cost, measured by the residual of sum squares (RSS).

\begin{figure}[!htb]
    \centering
    \includegraphics[scale=0.42]{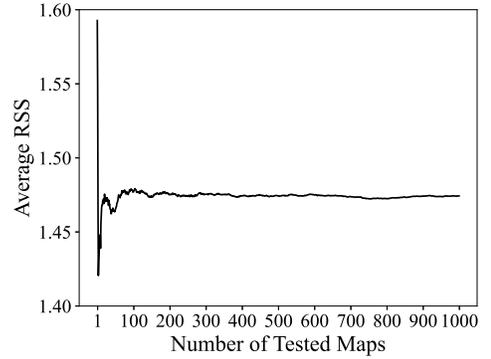}
    \caption{Average residual sum of squares as the number of basic scenarios increases.}
    \label{fig:statistics}
\end{figure}

\end{appendices}

\ifCLASSOPTIONcaptionsoff
  \newpage
\fi


%\vfill


\end{document}